\documentclass[11pt]{article}

\usepackage[letterpaper,margin=1in]{geometry}
\usepackage[T1]{fontenc}
\usepackage[utf8]{inputenc}
\usepackage{amsmath,amssymb,amsthm}
\usepackage{graphicx}
\usepackage{booktabs}
\usepackage[numbers]{natbib}
\usepackage{titling}
\usepackage{titlesec}
\usepackage{wrapfig}
\usepackage[colorlinks=true,allcolors=blue]{hyperref}

\usepackage[toc,page,header]{appendix}
\usepackage{minitoc}

\title{\bfseries
  Deep Generative Models for Phylogenetic Inference\\with Complex Evolutionary Processes
}

\author{
  Ethan Baron\textsuperscript{*}\quad \quad
  Alan N.~Amin\textsuperscript{*}\quad 
  \quad
  Andrew Gordon Wilson \\[3pt]
  \normalsize{New York University}
  \texttt{}
}

\date{}

\usepackage{microtype}
\usepackage{graphicx}
\usepackage{subcaption}
\usepackage{booktabs} 
\usepackage{hyperref}

\usepackage{algorithm}
\usepackage{algorithmic}

\usepackage{amsmath}
\usepackage{amssymb}
\usepackage{mathtools}
\usepackage{amsthm}
\usepackage{enumitem}
\usepackage{placeins}

\usepackage{bbm}
\usepackage[dvipsnames]{xcolor}

\definecolor{bu_purple}{RGB}{103, 93, 163}
\definecolor{mo_green}{RGB}{54, 154, 83}
\definecolor{td_blue}{RGB}{44, 120, 178}

\newcommand{\bu}{\textbf{\textcolor{bu_purple}{bottom-up}}}
\newcommand{\mo}{\textbf{\textcolor{mo_green}{middle-out}}}
\newcommand{\td}{\textbf{\textcolor{td_blue}{top-down}}}
\newcommand{\Bu}{\textbf{\textcolor{bu_purple}{Bottom-up}}}
\newcommand{\Mo}{\textbf{\textcolor{mo_green}{Middle-out}}}
\newcommand{\Td}{\textbf{\textcolor{td_blue}{Top-down}}}
\newcommand{\Root}[1]{\mathrm{root}}

\usepackage[capitalize,noabbrev]{cleveref}

\theoremstyle{plain}
\newtheorem{theorem}{Theorem}[section]
\newtheorem{proposition}[theorem]{Proposition}

\theoremstyle{definition}

\theoremstyle{remark}

\usepackage{setspace}
\usepackage{indentfirst}

\begin{document}

\maketitle
{\renewcommand{\thefootnote}{}%
 \footnotetext{\textsuperscript{*}Equal contribution. Correspondence to: \href{mailto:alan.nawzad.a@gmail.com}{\texttt{alan.nawzad.a@gmail.com}}}}
\vspace{-3.0em}

\begin{abstract}
\noindent
Inference of phylogenetic trees from sequences plays a central role in understanding evolutionary relationships, with applications ranging from tracking pathogen spread to reconstructing the history of life.
Conventionally, practitioners obtain posteriors over trees using likelihood-based methods like MCMC.
However, likelihoods are only available for simple evolutionary models with restrictive assumptions. 
For realistic evolutionary models, inference using conventional methods is prohibitively expensive, inaccurate, or even impossible.
In practice, practitioners settle for using as flexible a conventional method as possible, even though they expect it is misspecified.
In principle, we could use simulation-based inference: simulate tree and sequence data from an evolutionary model and train a neural network to predict the posterior over trees given sequences.
However it is a mathematical challenge to represent and train complex posterior distributions over trees.
Here we present a flexible family of conditional generative models over trees.
Our models iteratively build trees using three natural paradigms --- top-down, middle-out, or bottom-up ---
and are trained using a discrete diffusion framework.
For real and synthetic datasets of up to $65$ taxa and $>10^5$ tokens, our models recover MCMC posteriors much more accurately than the previous state-of-the-art simulation-based method Phyloformer 2.
Finally, we compare to IQ-Tree, a popular flexible-but-misspecified conventional method.
Our models match its tree estimates on simple data, but on more realistic data, our method is reliably more accurate.
\end{abstract}

\vspace{0.4em}
\onehalfspacing

\section{Introduction}\label{sec: intro}

To understand ancient evolution~\citep{Whelan2017-il}, the origins of pathogens~\citep{Bloom2025-hj}, and the function of the immune system~\citep{DeWitt2018-vw}, biologists collect DNA or protein sequences and try to infer a phylogenetic tree~\citep{Felsenstein2004-xi}.
In addition to the topology of the tree, biologists wish to infer scalar parameters that describe, for example, when a most recent common ancestor has occurred~\citep{Tamura2012-uf}, whether there is selective pressure~\citep{Goldman1994-az}, or how fast a pathogen is growing~\citep{Drummond2005-mv}.
Just as important as inferring the tree is representing our uncertainty~\citep{Huelsenbeck2000-ef}.
Bayesian inference provides a natural framework: the posterior represents our belief in the tree, including our uncertainty~\citep[Ch. 18]{Felsenstein2004-xi}.

To sample from the posterior, in principle we can use MCMC~\citep{Huelsenbeck2001-mo, Suchard2018-bd}.
However, Bayesian inference tools like MCMC require access to a likelihood, which is only easy to calculate for the simplest models of evolution.
More complex models of evolution either (A) have tractable but expensive likelihoods or (B) completely intractable likelihoods.

In models with expensive likelihoods, (A), inference with MCMC is prone to errors~\citep{Whelan2017-gv} and time consuming, in some cases taking weeks to infer a tree with a hundred species~\citep{Susko2018-nz}.
These models include those accounting for different mutation processes at different sites \citep{Banos2024-tz, Lartillot2004-ri}, changes of mutation processes over time \citep{Wang2007-xa, Crotty2020-kh}, and different mutation rates at different sites \citep{Yang1994-iq} and their correlation between nearby sites~\citep{Felsenstein1996-kl}.
In practice, practitioners speed up inference with strategies 
like taking the values of evolutionary parameters from previously fit models~\citep{Le2008-ht, Le2012-kh}, 
alternate fitting evolutionary parameters and the tree~\citep{Prillo2023-lq}, 
or relying on simpler models. Unfortunately,
these strategies are the sources of many errors in the literature~\citep{Pisani2015-io, Susko2018-nz, Youssef2020-kt}.

Other models, (B), have intractable likelihoods despite describing ubiquitous biological processes.
These include models of sequence insertions and deletions~\citep{Fletcher2009-yg}, selection-dependence between sites~\citep{Youssef2020-kt}, and context-dependent mutation rates~\citep{Seplyarskiy2021-nl}.
Such models are often dealt with by making strong assumptions~\citep{Holmes2001-xa, Matsen2022-go}, discarding data~\citep{Siepel2004-cn, Tan2015-lg}, or, more often, ignoring these phenomena entirely.

Unlike likelihood calculation, simulation is simple in nearly all models of evolution, including those with intractable likelihoods.
Ideally we could learn to perform inference by training a model on data simulated from the prior~\citep{Papamakarios2016-tp}, a framework known as simulation-based inference.
Indeed, this framework has been extremely successful for Bayesian inference in astronomy~\citep{Dax2021-dy}, neuroscience~\citep{Goncalves2020-hn}, and on tabular data~\citep{Hollmann2025-bv}.
The main challenge in applying simulation-based inference to this problem is representing a posterior distribution over the combinatorial space of tree topologies, a space of size $(2N-3)!!$.
Indeed, inference for scalar evolutionary parameters is mature~\citep{Fourment2024-tk}, so here we focus on the hard problem of inferring posteriors over trees.

Previous models use a single forward pass to produce a parametric distribution over trees, limiting expressivity~\citep{Blassel2025-vq}.
In this paper, to build flexible models that can support rich posterior distributions, we design iterative generative models that perform multiple forward passes to sample a tree topology from the posterior. 
In Sec.~\ref{sec:flexible}, we describe three natural paradigms for parameterizing iterative generative models for trees: \bu, \mo, and \td.
Since phylogenetic trees can be represented as graphs, we parameterize our models as graph neural networks, leveraging methods from the graph foundation model community \citep{wang2025graphfoundationmodelscomprehensive};
    we also introduce new architectural ideas to condition on a tree and avoid representation collapse on huge data.
In Sec.~\ref{sec:training}, we show we can train large-scale models for each paradigm using a discrete diffusion objective.
After training on only 100 GPU hours, our models accurately fit posteriors of trees for data with up to $65$ DNA or protein sequences of length $>2500$, mapping $>10^5$ input tokens to a posterior over $>10^{100}$ topologies.

We show in Sec.~\ref{sec:results} that our generative models learn to represent an accurate and calibrated posterior on real and synthetic data in cases where Bayesian inference is tractable, substantially outperforming previous simulation-based models.
Moreover, we show that our models successfully learn posteriors under complex evolutionary models where Bayesian inference is intractable, significantly outperforming the conventional strategy of using misspecified models.

Our code is available at \url{https://github.com/AlanNawzadAmin/DeepPhylo}.

\section{Background and motivation}
\label{sec:background}

In this section, we introduce Bayesian phylogenetic inference, explain why conventional methods cannot be used under realistic evolutionary models, and advocate for simulation-based inference.

\subsection{Overview of phylogenetic inference}

Let $X_1, \ldots, X_N \in V^L$ be a set of DNA or protein sequences of length $L$ from the genomes of $N$ species (also known as ``taxa'') descended from a common ancestor, where $V$ is the alphabet of nucleotides or amino acids.
We represent how these sequences evolved with a phylogenetic tree $T$ with $N$ leaves, one representing each species. In a phylogenetic tree, each internal node represents the common ancestor of its two children. Phylogenetic trees are bifurcating, meaning each internal node has exactly two children.

In Bayesian phylogenetic inference, we wish to sample from the \emph{posterior} $p(T \mid X_{1:N})$ under some \emph{prior} base distribution over tree topologies $p(T)$ and \emph{likelihood} corresponding to an evolutionary model $p(X_{1:N} \mid T)$. 
As explained in Sec.~\ref{sec: intro}, we consider inferring the topology of the tree $T$ without predicting branch lengths, though our methods can be easily adapted to accommodate branch length prediction.

In some settings, the root of the tree is not identifiable from sequence data, so $T$ is an unrooted tree.
Throughout this paper we consider a rooted $T$ and the nonidentifiability of the root will manifest as its uncertainty in the posterior.

\subsection{Priors over tree topologies}

There are several classical choices for the prior over tree topologies $p(T)$. The PDA distribution (also known sometimes as simply the uniform model) assigns a uniform probability to each topology, typically producing ``unbalanced'' trees.
In contrast, the Yule model results from assuming a uniform rate of speciation at all nodes, producing relatively balanced trees.
In practice, published trees are generally less balanced than Yule trees but more balanced than PDA trees. Ford's alpha model  \citep{Ford2005-gf} and Aldous' beta model \citep{Aldous_beta} offer two ways to interpolate between these extremes.

\subsection{Evolutionary models}

The simplest evolutionary model is the Jukes-Cantor (JC) model~\citep{Jukes1969-ap}.
The JC model simulates an independent Markov substitution process at each position of the sequence for each species.

Formally, JC parameterizes the evolution of species along a tree $T$ with given branch lengths $r_{u\to v}$ for all parents and children $u, v$ in $T$. We start with a single species $Y$ at the root, whose sequence is sampled uniformly at random $Y_{\Root{}}^{(\ell)}\sim\mathrm{Uniform}\{A, T, C, G\}$ for all $\ell = 1, 2, \ldots, L$. For each branch, we then simulate an independent uniform Markov substitution process for time $r_{\Root{} \to v}$, mutating $Y^{(\ell)}$ into a different random letter every $\Delta t \sim\mathrm{Exp}(1)$. We continue similarly, simulating the Markov process down the tree $T$ until we have simulated the sequences at all of the leaves $X_1, \dots, X_N$.
\label{sec:jc}

Despite the simplicity of the JC model, it is not obvious how to calculate $p(X_{1:N}\mid T)$, since it requires marginalizing out all ancestral sequences. For this calculation, \citet{Felsenstein1973-ro} developed the pruning algorithm, which continues to be the backbone for likelihood calculations in phylogenetic inference. Felsenstein's pruning algorithm can also be applied to calculate likelihoods for some more elaborate models, such as ones that allow different sites to mutate with different rates. However, it strongly relies on the assumption that evolution occurs independently at each position.

Say, when simulating the JC model, the rate of mutation of a letter $Y^{(l)}$ was not independent, but rather depended on the identity of its neighbors $Y^{(l-1)}$ and $Y^{(l+1)}$. Alternatively, say there was a chance for a random chunk of five contiguous letters to be inserted or deleted. These phenomena are ubiquitous in real data, and while it is straightforward to simulate $X_{1:N}\mid T$, computing the likelihood becomes intractable, making typical strategies for posterior inference impossible.

\subsection{Simulation-based posterior inference}

When we use simple evolutionary models like JC, we can compute likelihoods and perform Bayesian inference with MCMC. However, when likelihood evaluation is expensive or impossible, we cannot perform Bayesian inference with these methods. How can we perform posterior inference when we cannot compute the likelihood $p(X_{1:N} \mid T)$, but we can only sample from this distribution?

We suggest \emph{simulation-based inference}~\citep{Papamakarios2016-tp}, in which we sample many synthetic examples of $(X_{1:N}, T)$ pairs following a given prior and likelihood, and train a generative model to predict $T$ conditioned on $X_{1:N}$. The trained model can then sample from the posterior for $T$ given real sequences $X_{1:N}$!
To do so, we need a scalable, model which can learn a full posterior $p(T\mid X_{1:N})$ without access to the likelihood.

\section{Related works}
\label{sec:rel_work}
There are many works that apply machine learning methods for Bayesian phylogenetic inference. 
Most of these methods seek to improve upon MCMC methods for Bayesian inference, which are applicable when likelihoods are tractable.
These include Hamiltonian Monte Carlo~\citep{Dinh2017-zp}, sequential Monte Carlo~\citep{Wang2015-md, Moretti2021-yg, Chen2025-yg}, variational inference~\citep{Zhang2024-jz}, and pseudo-likelihood methods~\citep{Prillo2023-lq}.
Recent work however has considered likelihood-free inference.
\citet{Suvorov2020-fe} considered learning to map from $X_{1:N}$ to a single ``most-correct'' tree $T$.
Finally \citet{Blassel2025-vq} suggested learning the posterior distribution $p( T \mid X_{1:N})$.

\textbf{Can we apply generative models in the literature to our problem?}
Although many generative models over trees have been proposed for variational likelihood-based inference (VI)~\citep{Zhang2024-jz},
summarizing distributions over trees~\citep{Xie2025-nv},
or likelihood-free inference~\citep{Suvorov2020-fe}, almost all existing methods cannot be used for our problem since
\begin{itemize}[noitemsep]
    \item \textbf{they do not learn a full posterior. }
    Methods that predict quartets~\citep{Suvorov2020-fe, Su2025-sa} or distance matrices~\citep{Jiang2023-lx, Nesterenko2024-fm}; or use a heuristic loss~\citep{Zhang2025-bo} only make point predictions.
    \item \textbf{they do not scale to our setting.} It is not clear how to efficiently learn energy-based models~\citep{Zhang2023-pc} 
    for datasets larger than a few short sequences.
    And the number of parameters of a split or subsplit network is at least $2^N$~\citep{Larget2013-kx, Zhang2018-fv, Zhang2024-jz};
    when applying these models practitioners restrict support to the small number of splits that were observed in MCMC posteriors, but we don't have that option when we're interested in supporting the full combinatorial space of tree topologies.
    \item \textbf{they require a tractable likelihood to train.}
    Integrating the continuous encodings of ~\citet{Mimori2023-sf} and~\citet{Duan2024-dv} and training subsplit networks with restricted support~\citep{Zhang2018-fv} is only possible when optimizing a reverse KL.
    Also, the GFlowNet framework of \citet{Zhou2024-xo} requires a reward model.
    \item \textbf{or they don't condition on sequences.} ARTree \cite{Xie2023-ed} and ARTreeFormer \cite{Xie2025ARTreeFormer} summarize MCMC posteriors for a single, \textbf{fixed} set of sequences $X_{1:N}$. As a result their architectures do not condition on sequences.
    And, even with a conditional architecture, training using their paradigm leads to a worse model (see App.~\ref{app:art_limits}).
\end{itemize}
This leaves, to our knowledge, only Phyloformer 2~\citep{Blassel2025-vq} as a method appropriate for likelihood-free Bayesian inference.
However its flexibility is restricted as it only performs one forward pass to represent a distribution of trees (see App.~\ref{app:pf_limits} for details).

\section{Flexible generative models over trees}
\label{sec:flexible}

To perform simulation-based inference for phylogenetic trees, we need a flexible model capable of representing complex posterior distributions over tree topologies. 
In this section, 
we start by introducing three natural paradigms for constructing generative models over trees
and argue why we should use generative models that use many forward passes to sample a tree.

\begin{figure*}[htb]
    \centering
    \includegraphics[width=0.95\linewidth]{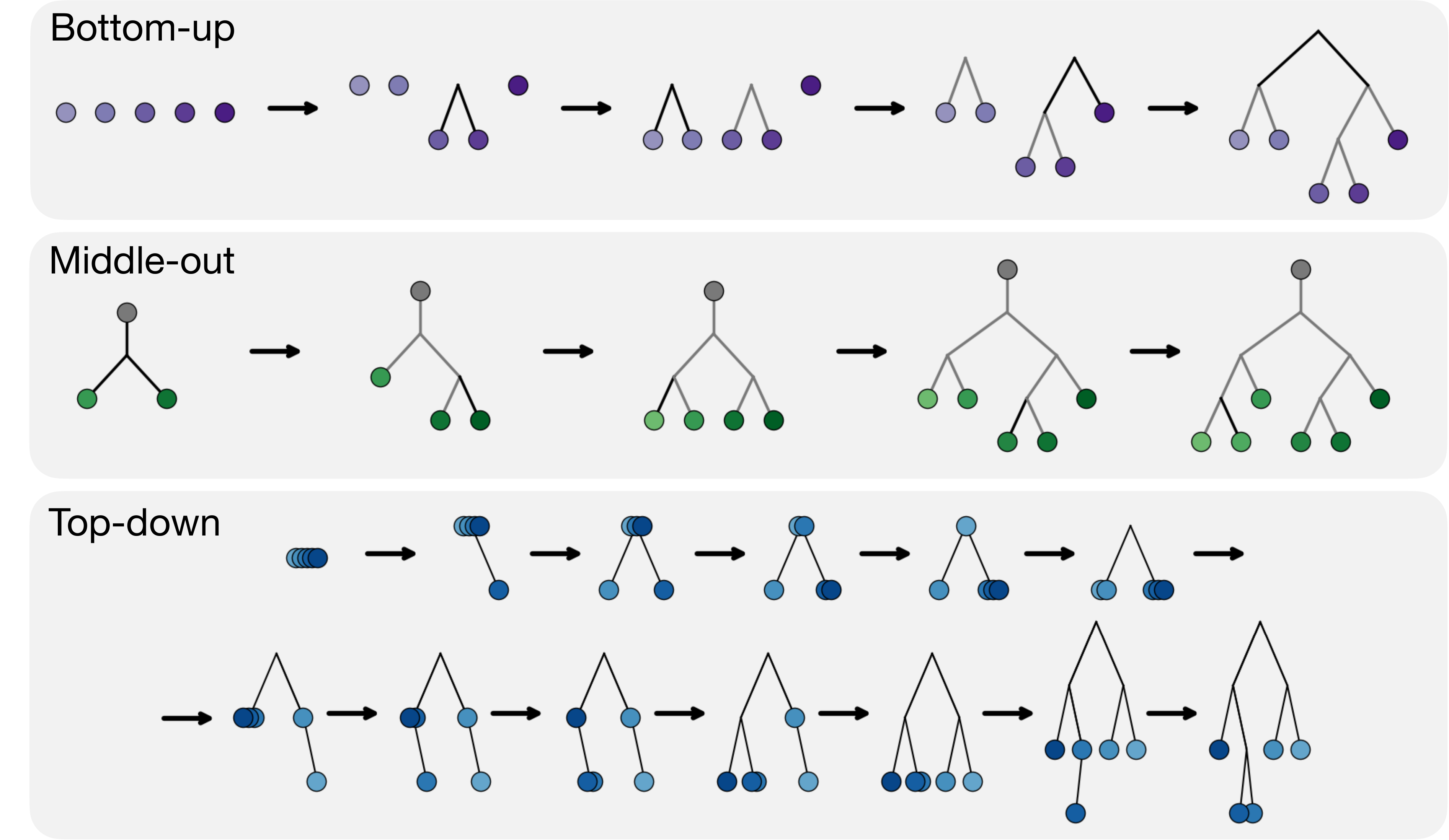}
    \caption{Illustration of three generative paradigms for trees. Each colored circle represents one taxon. For \bu, we start with each taxon as a separate subtree and repeatedly merge two subtrees until a single bifurcating tree is formed. Dark edges indicate which subtrees were merged in each step.
    For \mo, we start with a tree containing two taxa and add one taxon at a time. The gray circle depicts the ``pseudo-taxon'' representing the root of the tree, and dark edges show which taxon was added in each step.
    For \td, we start with all taxa at the root and percolate them downwards one step at a time until we reach a full bifurcating tree.}
    \label{fig:traj}
\end{figure*}

\subsection{Three generative paradigms for trees}

We consider three popular and natural paradigms for parameterizing generative models over trees: \bu, \mo, and \td.
We provide an illustration of the three paradigms in Fig.~\ref{fig:traj}.
We discuss alternative paradigms for generative models over trees and their challenges in App.~\ref{app: other tree models}.

\begin{itemize}
    \item The \bu{} paradigm generates trees by starting with each taxon as its own subtree containing one node, and iteratively merging two subtrees together until a complete bifurcating tree is formed.
    \item The \mo{} paradigm generates trees by adding one taxon to the tree at a time. Specifically, we start by randomly selecting two species and placing them in an unrooted tree, along with a ``pseudo-taxon'' representing the root of the original tree. We then graft one taxon at a time, attaching it to the current unrooted tree, until all taxa have been added.
    \item The \td{} paradigm generates trees by percolating taxa down from the root one layer at a time. Specifically, we start with all taxa at the root. At each step, we select one taxon with minimal depth and predict whether to percolate it down to the left or right. We stop percolating when a full bifurcating tree is formed.
\end{itemize}

All three of these paradigms are explored in the mathematical phylogenetics and generative modeling literature:
for example, \bu{} has neighbour joining~\citep{Saitou1987-js} and Phyloformer 2 \citep{Blassel2025-vq};
\mo{} has Ford's alpha model~\citep{Ford2005-gf} and ARTree \cite{Xie2023-ed};
and \td{} has Aldous' beta model~\citep{Aldous_beta} and subsplit networks~\citep{Zhang2018-fv}.
In none of these paradigms however is there a flexible conditional generative model over trees.

\subsection{Achieving flexibility with iterative generative models}

How can we develop generative models that are sufficiently flexible to represent arbitrary posterior distributions over tree topologies? We advocate for using \emph{iterative generative models} that perform multiple forward passes through the model to gradually construct a tree one step at a time. Unlike past methods that use a single forward pass to parameterize the full posterior, performing multiple forward passes allows the model to represent arbitrarily complex distributions by conditioning on the intermediate state of the tree before predicting the next step.

Formally, we start the generative process with a structure $T_M$ that depends only on the sequences, without any information about the real tree. We then use a model $q_\theta$ to build up the tree one step at a time by repeatedly sampling from the distribution $q_\theta(T_{m-1} \mid T_m, X_{1:N})$. We continue this process for $M$ steps until we have generated the full tree $T_0$\footnote{For \mo{} we could either start with $T_M$ containing two random sequences, or take $T_M$ to be empty and have the first two steps $q_\theta(T_{m-1} \mid T_m, X_{1:N})$ randomly add two sequences. We take the latter convention.}.
\begin{figure*}[tb]
    \centering
    \includegraphics[width=0.85\linewidth]{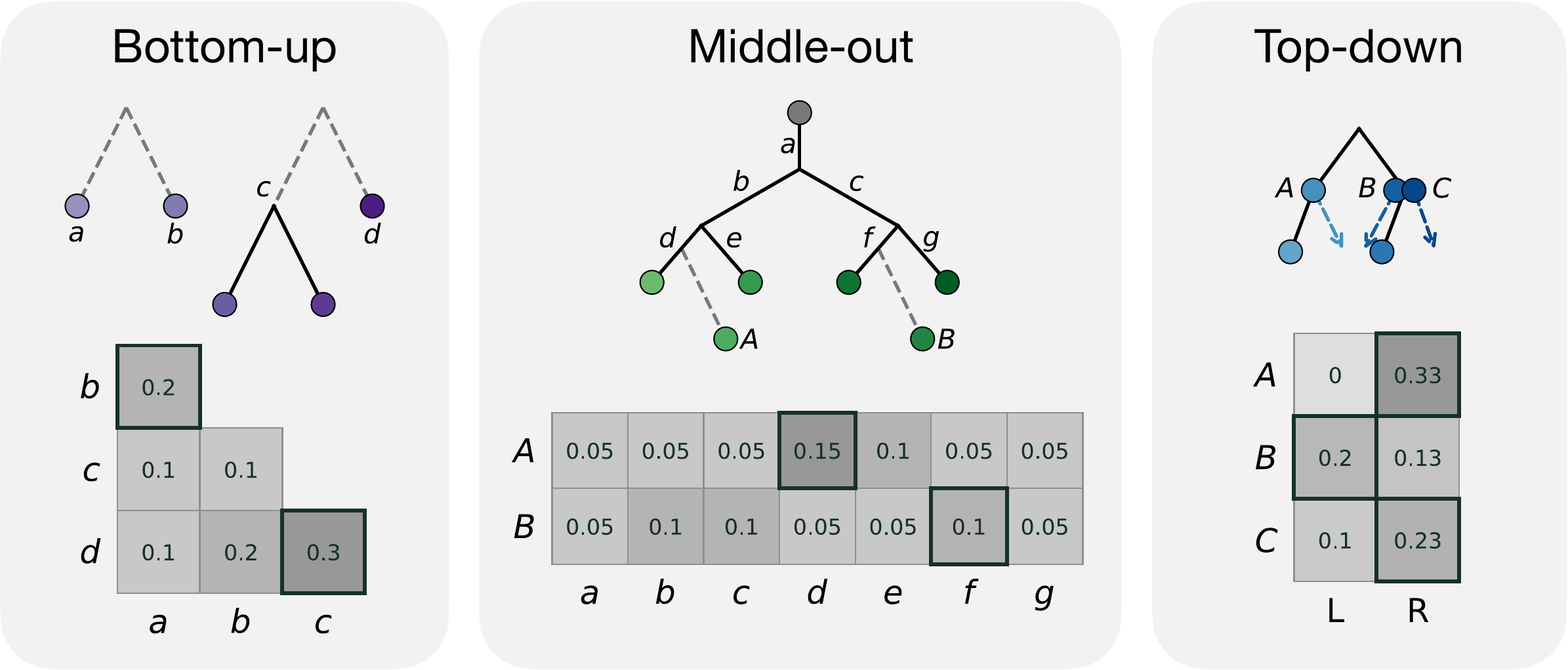}
    \caption{Illustration of sample model predictions for the three generative paradigms.
    For \bu{}, the dashed edges indicate the valid merges consistent with $T_0$, corresponding to the highlighted cells in the matrix of model predictions.
    For \mo{}, the dashed edges indicate the correct attachment edges consistent with $T_0$, corresponding to the highlighted cells in the matrix of model predictions.
    For \td{}, the arrows indicate the correct percolation directions consistent with $T_0$, corresponding to the highlighted cells in the matrix of t predictions. Taxon $A$ is constrained to percolate right as $T_0$ is a bifurcating tree.}
    \label{fig:parameterization}
\end{figure*}
Fig.~\ref{fig:parameterization} illustrates the shape of the model's predictions for a single step of generation under each of the paradigms. 
\begin{itemize}
    \item For the \bu{} paradigm, at each step, the model takes in a set of $m+1$ subtrees, and predicts a distribution of size ${m+1 \choose 2}$ over the next two subtrees to merge.
    \item For the \mo{} paradigm, at each step the model conditions on the current tree and predicts where to attach each missing taxon using a distribution over the tree's edges. 
    \item For the \td{} paradigm, at each step the model conditions on the current state of the tree and predicts the probability of each taxon percolating left or right. 
\end{itemize}

Interestingly, a model that naively outputs uniform predictions leads to a different distribution over tree topologies for each of the three paradigms. For \bu{} we obtain the Yule distribution, for \mo{} we obtain the PDA distribution, and for \td{} we obtain the ``almost completely balanced'' distribution described in \citep{Aldous2001-ec}. This might suggest that our choice of generative paradigm depends on our desired prior over tree topologies. However, we show in App.~\ref{app:expressivity} that all three paradigms can support a full spectrum from PDA to Yule with simple adjustments to the naive uniform predictions, meaning all are sufficiently flexible for our task.
\label{sec:expressivity}

\section{Efficiently training our iterative generative models}
\label{sec:training}

In this section, we describe how to efficiently train iterative generative models over trees using a discrete diffusion framework, explain our choices for a ``noising process'' in each paradigm, and provide details on our implementation.

\subsection{Training iterative generative models with discrete diffusion}
\label{sec:diffusion}

We want to train our models to maximize the likelihood of the simulated data. 
Letting $T_M, T_{M-1}, \ldots, T_1$ be the states of the tree while generating $T_0$, the likelihood is:
\begin{equation}
q_\theta(T_0 \mid X_{1:N}) =
\sum_{T_1, T_2, \ldots, T_M} \prod_{m=1}^M q_\theta(T_{m-1} \mid T_m, X_{1:N})
\label{eqn:likelihood}
\end{equation}
Unfortunately, directly optimizing such a likelihood is intractable as it requires marginalizing over all possible paths that reach the tree $T_0$. 

Instead, we propose training the model $q_\theta$ using the discrete diffusion framework. 
Specifically, we define a noising process that generates paths $p(T_1, \dots, T_M\mid T_0)$ by noising $T_0$; then we train our model to reverse these paths.
We do so by optimizing an ELBO of the form:
\begin{equation}
\log q_\theta(T_0 \mid X_{1:N})
\geq
- M\mathbb E_{m\sim\mathrm{Unif}(1, \dots, M), T_m \sim p\left(T_m \mid T_0 \right)} \left[ D_\textrm{KL}(p(T_{m-1} \mid T_m, T_0) \, || \, q_\theta(T_{m-1} \mid T_m, X_{1:N})\right]
\label{eqn:elbo}
\end{equation}
Training the model now involves simulating a training example $(T_0, X_{1:N})$, applying the noising process $p(T_m \mid T_0)$ for some number of noising steps $m$, and pushing the model's predictions on $T_m$ closer to the true reverse process that recovers $T_0$.
The ELBO is maximized when $q_\theta(T_{m-1} \mid T_m, X_{1:N})=\mathbb E_{p(T_0\mid X_{1:N}, T_m)}p(T_{m-1} \mid T_m, T_0)$at which point $q_\theta=p$.
We review the derivation of the ELBO and cover related work on discrete diffusion in App.~\ref{app:diffusion}.

\subsection{Rao-Blackwellizing with a closed-form loss}

To evaluate the KL in Eqn.~\ref{eqn:elbo}, we need $p(T_{m-1}\mid T_m, T_0)$.
Alternatively however, we may be tempted to write the ELBO
\begin{equation}\label{eqn: naive elbo}
\begin{aligned}
    - \sum_{m=1}^M \mathbb{E} \textrm{KL}(p(T_{m-1} \mid T_m, T_0) \, || \, &q_\theta(T_{m-1} \mid T_m, X_{1:N}))\\
    =&-\sum_{m=1}^M \mathbb{E}_{T_0, T_1, \dots, T_m}\log q_\theta(T_{m-1} \mid T_m, X_{1:N})+\mathrm{const}.
\end{aligned}
\end{equation}
\begin{wrapfigure}{l}{0.4\linewidth}
    \centering
    \vspace{-0.5cm}
    \includegraphics[width=0.99\linewidth]{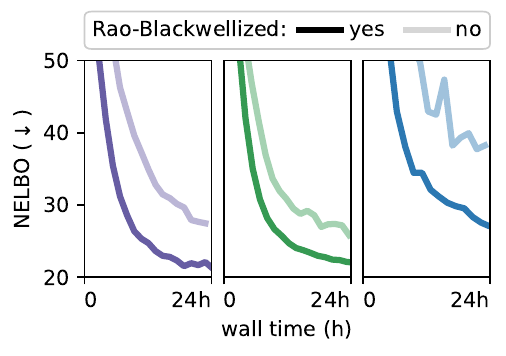}
    \caption{\textbf{Our closed form losses are crucial for training.}
    On the ``complex'' dataset described below we train models with or without Rao-Blackwellizing the KL estimate in Eqn.~\ref{eqn: naive elbo}.
    \vspace{-0.3cm}
    }
    \label{fig:RB ablation}
\end{wrapfigure}and train $q_\theta$ by simulating from $T_0$ up to $T_m$ and predicting the $T_{m-1}$ that was visited along that simulation.
Or we could get a similar loss by picking a single ``correct'' path for each $T_0$ (App.~\ref{app:pf_limits} and ~\ref{app:art_limits}).
One only needs to be able to simulate trajectories to train this way.

In principle however, with a closed form $p(T_{m-1} \mid T_m, T_0)$ one could evaluate the entire KL, integrating over all paths and thereby reducing the variance of the estimate, i.e. Rao-Blackwellizing.
In Fig.~\ref{fig:RB ablation} we see this in practice as well: using the formulas we derive in the next section, our Rao-Blackwellized training acheives substantially better ELBOs.
In App.~\ref{app:pf2 ablations} we show a similar result for the canonicalizing strategy used in previous work.

\subsection{Noising process for generative models over trees}
\label{app:noising}

To train a discrete diffusion model using Eqn.~\ref{eqn:elbo}, we must define a noising process $p(T_{1:M} \mid T_0, m)$ carefully so that the target distribution $p(T_{m-1} \mid T_m, T_0)$ is computable closed-form.

For \bu, at the start of the noising process, we label the root node as ``exposed''. Then, each noising step involves removing an exposed internal node from the current graph, ``exposing'' its two children. We select the next node to remove in each step as:
\begin{align*}
p_\text{BU}(T_m=\texttt{remove}&(T_{m-1}, \texttt{node}) \mid T_{m-1})
\propto |\texttt{descendants}(\texttt{node})|
\end{align*}
We multiply by the number of descendants of $\texttt{node}$, not counting $\texttt{node}$, to get a tractable form for $p(T_{m-1} \mid T_m, T_0)$ involving two tree traversals to calculate described in App.~\ref{app:rb_bu}.

For \mo, each noising step involves pruning one taxon from the tree and collapsing the edges connecting its two siblings. Formally, we have, for a taxon $v$:
\begin{align*}
p_\text{MO}(T_m=\texttt{prune}(T_{m-1}, v) \mid T_{m-1}) = \frac{1}{N - m + 1}
\end{align*}
Here, the distribution $p(T_{m-1} \mid T_m, T_0)$ is again uniform over all valid attachments in $T_m$ that are consistent with $T_0$. Note that for a given taxon, there is only one attachment edge in $T_m$ that is consistent with the ground truth $T_0$. Note also that our target has multiple non-zero entries corresponding to each of the taxa not in the current tree.
\label{sec:mo_rb}

For \td{}, each noising step involves selecting a taxon of maximal depth and percolating it up one layer to its parent. The noising process continues until all taxa reach the root node, at $T_M$. Formally, for a tree $T_{m-1}$ with maximum depth $d$, we have:
\begin{align*}
p_\text{TD}&(T_m=\texttt{percolate\_up}(T_{m-1}, \texttt{taxon}) \mid T_{m-1}) 
= \begin{cases} 0, \quad &\texttt{depth}(\texttt{taxon}) < d \\
\frac{1}{\texttt{\# of taxa at depth} \, d}, \quad &\textrm{otherwise}
\end{cases} 
\end{align*}
The distribution $p(T_{m-1} \mid T_m, T_0)$ is uniform over all percolations in $T_m$ that are consistent with $T_0$.

\paragraph{Noise schedule} 
In the ELBO from \cref{eqn:elbo}, we use a uniform distribution over the number of noising steps $m$. However, in practice some noising steps may be much easier for the model to learn than others. Ideally, to train as efficiently as possible, we would like to sample training examples according to the difficulty, by adapting the noise schedule. In \cref{app:noise_schedule}, we describe how to build a noise schedule mirroring the mutual information schedules of \citet{Austin2021-dg}.
In App.~\ref{app: ablate math} however we show this makes little difference in practice, so future practitioners may proceed with the uniform noise schedule without having to implement or derive a mutual information schedule.

\subsection{Architecture}

\begin{wrapfigure}{r}{0.4\linewidth}
    \centering
    \includegraphics[width=0.99\linewidth]{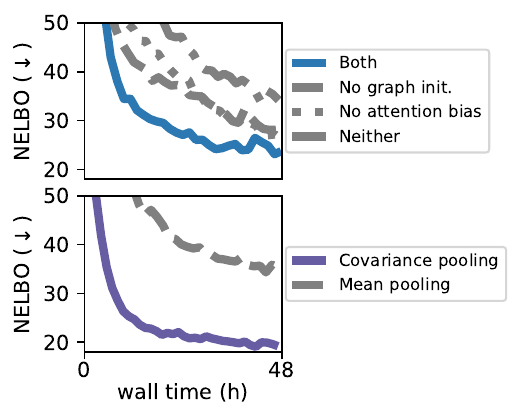}
    \caption{\textbf{Our strategies for conditioning on the graph and covariance pooling are crucial for training.}
    We ablate (Top) our graph conditioning strategies and 
    (Bottom) our covariance pooling.
    Each is shown on an example paradigm: \td{} and \bu{}.
    }
    \label{fig:graph ablation}
\end{wrapfigure}
We adapt the EvoPF architecture from \citep{Blassel2025-vq}. 
EvoPF embeds $N$ sequences of length $L$, $X_{1:N}$, into a tensor of size $[N, L, 128]$ and performs axial attention over the dimensions $N$ and $L$ separately.
Finally it performs a mean across the length dimension to get a matrix of $[N, 128]$ embeddings.
We make two modifications.
We provide full details in App.~\ref{app:architecture} and ablations showing these two modifications were crucial for performance in Fig.~\ref{fig:graph ablation} (details in App.~\ref{app: arch ablations}).

First we need to condition on the current state of the graph $T_m$.
    So we initialized species embeddings with graph topology information, and biased attention among species using graph distance. 

Second, near initialization, the mean over $L>1000$ positions returns nearly identical embeddings across taxa:
    if the embedding at each position represents only its nucleotide, then the mean is only the base composition, nearly the same for every taxon.
    We developed a new architecture we call \textbf{covariance-pooling} which embeds the average similarity between pairs of nodes:
    if embeddings represent nucleotides, now we get the Hamming similarity.

Also, as noted in Sec.~\ref{sec:expressivity}, each paradigm is capable of representing prior distributions ranging from Yule to PDA. 
We initialize at the desired prior by multiplying the neural network outputs with these weights.
In App.~\ref{app: ablate math} we show our models can learn these weights even without this initialization.

\section{Experiment details}

We now describe our experimental setup, including training data, architecture, and training details.

\subsection{Simulated training data}

We train our models on synthetic data simulated following a prior $p(T)$ and evolutionary model $p(X_{1:N} \mid T)$. 
We set $p(T)$ to be a PDA prior on trees to match the priors in MrBayes.

We consider three types of evolutionary models $p(X_{1:N} \mid T)$ for our experiments. First, we train our models on the simplest model of DNA evolution (\textbf{JC}). As described in Sec.~\ref{sec:jc}, the JC model yields tractable likelihoods, allowing us to use MrBayes to obtain gold standard MCMC samples from the posterior. Next, we consider two more realistic evolutionary models: a more \textit{complex model} accounting for natural selection in proteins (\textbf{CAT}) and a model of DNA evolution that has \textit{intractable} likelihoods because of the presence of insertions and deletions (\textbf{InDel}).

We train our models to support datasets with between $N=25$ and $N=65$ species, and sequences of length $L=300$ to $L=2600$, based on the sizes of 8 classical evolutionary datasets~\citep{Zhang2018-fv}. Our models can therefore be immediately applied to most evolutionary and immunological inference tasks on 1-5 genes. We discuss scaling to larger $N$ and $L$ in Sec.~\ref{sec:conclusion}. We provide further details on our simulated training datasets in App.~\ref{app:evolutionary_models}.

\begin{figure*}[tb]
    \centering
    \includegraphics[width=0.9\linewidth]{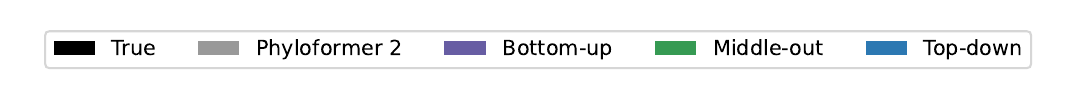}
    \begin{subfigure}{0.49\linewidth}
        \centering
        \includegraphics[
            width=\linewidth,
            trim=0 0 118pt 0,
            clip
        ]{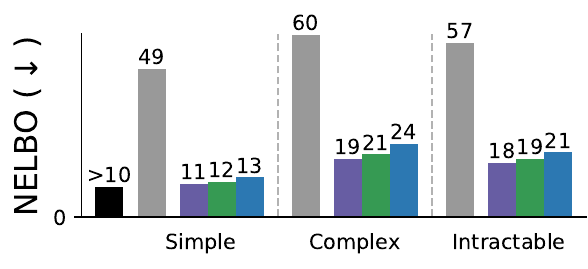}
        \caption{ELBOs on training data}
        \label{fig:elbos}
    \end{subfigure}
    \hfill
    \begin{subfigure}{0.49\linewidth}
        \centering
        \includegraphics[
            width=\linewidth,
        ]{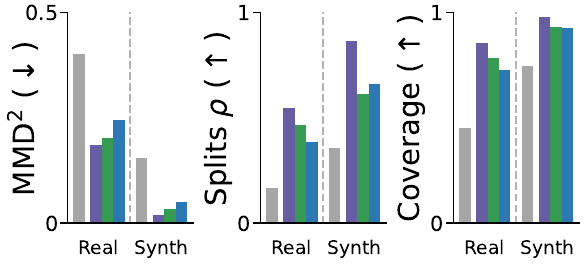}
        \caption{Posterior calibration on real and synthetic datasets}
        \label{fig:mmd}
    \end{subfigure}
    \caption{\textbf{(Left)}  Our models obtain substantially better negative ELBOs than Phyloformer 2.
    The ELBO of a trivial model that scores all trees uniformly is 157 nats.
    \textbf{(Right)}  Our models reconstruct MCMC posteriors on real datasets more accurately than Phyloformer 2.
    \vspace{-0.3cm}}
    \label{fig:elbos-side-by-side}
\end{figure*}

\subsection{Training details}

We use Adam with $\mu$P-rescaled per-layer learning rates~\citep{Yang2022-sy}
(base $\mathrm{lr}=0.01$), CompleteP scaled residual layers~\citep{Dey2025-uv}, $1{,}000$-step linear warmup,
EMA $0.9998$, gradient-norm clip $500$, \texttt{bf16-mixed} precision, and an
effective batch size of $32$ trees per step.
Each model trains for 48 hours on $2$ A100 GPUs, which is approximately $800,000$ datapoints for each paradigm.
Our experiments with Phyloformer 2 are controlled for compute: we train their models according to their default hyperparameters for 48 hours on $2$ A100 GPUs, which is approximately $1,100,000$ datapoints.
We also reduce the pair-attention dimension in our EvoPF architecture from $256$ to $128$, giving our models $13.6$ million parameters compared to $20.0$ million for our baseline, Phyloformer 2.

\section{Results}
\label{sec:results}

In this section, we provide our empirical results. 
We compare against Phyloformer 2, the previous state-of-the-art method for simulation-based inference.
First, we show that our models learn to fit simulated data much better than Phyloformer 2, as measured by ELBOs. 
Second, we show that our models match gold standard MCMC posteriors more accurately than Phyloformer 2. 
Lastly, we show that our models accurately infer the correct tree for the complex and intractable evolutionary models, improving accuracy over the best existing alternative.

We remind the reader that real evolutionary data almost never has a ``true'' tree that we can observe.
Our ground truth always depends on the practitioners' prior $p(T, X_{1:N})$ and our goal is to show that we can perform accurate posterior inference even for priors that have been traditionally challenging to model.

\subsection{Our models are state-of-the-art for simulation-based phylogenetic inference}

In Fig.~\ref{fig:elbos}, we compare the ELBOs achieved by our models against Phyloformer 2. We find that across datasets and generative paradigms, our models consistently fit the data much better than Phyloformer 2.
These results confirm that our flexible parameterization and efficient training via the discrete diffusion ELBO significantly improve the performance of simulation-based inference for phylogenetics. More generally, our work is the first to successfully apply the \td{} and \mo{} paradigms to this problem, and we show that both of these paradigms can far surpass the performance of the previous state-of-the-art.
In App.~\ref{app:pf2 ablations} we show through ablations that our improvement over Phyloformer 2 comes from our new training objective and the flexibility of our framework rather than a pessimistic ELBO or architectural changes.

For the JC data we are able to estimate the true average entropy of the data using MCMC (App.~\ref{app: JC ML}) at $>10$ nats, which is the uncertainty of a uniform distribution over $e^{10}\approx 20,000$ trees.
This demonstrates the significant uncertainty in the true tree even in the simplest synthetic settings, highlighting the importance of fitting a posterior distribution.
We see our models have an average KL of $<3$ nats (difference between our ELBO and the entropy) suggesting an absolute accurate fit.

On these data \bu{} consistently outperformed \mo{} which in turn outperformed \td{}.
This could be because the many noising steps of \td{} increase the variance of the loss during training.
The superior performance of \bu{} may also reflect the bottom-up nature of the likelihood calculation on simple phylogenetic data~\citep{Felsenstein1973-ro}.

\begin{figure*}[tb]
    \centering
    \includegraphics[
        width=0.9\linewidth,
        trim=0pt 0pt 0pt 0pt,
        clip
    ]{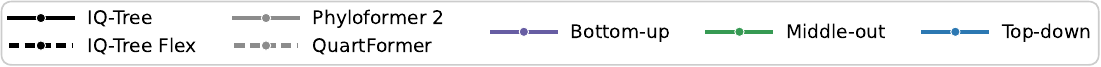}
    \begin{subfigure}{0.49\linewidth}
        \centering
        \includegraphics[
            width=\linewidth,
            trim=0 0 130pt 0,
            clip
        ]{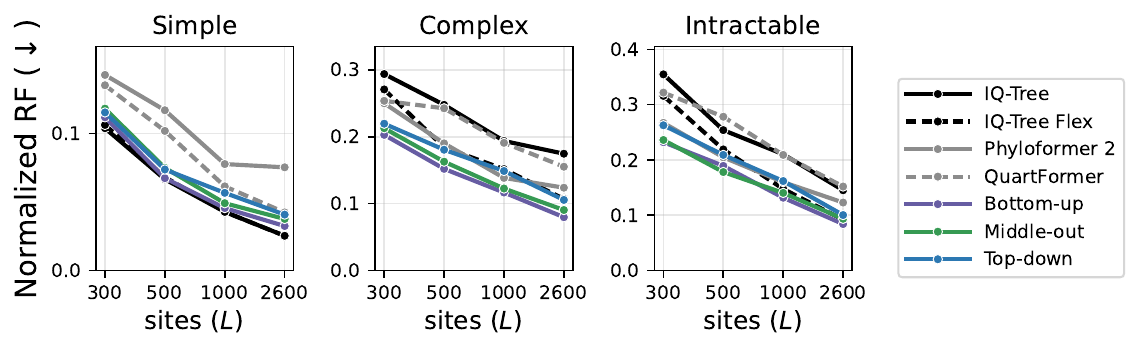}
        \caption{$N = 45$}
        \label{fig:rf_n}
    \end{subfigure}
    \hfill
    \begin{subfigure}{0.49\linewidth}
        \centering
        \includegraphics[
            width=\linewidth,
            trim=0 0 130pt 0,
            clip
        ]{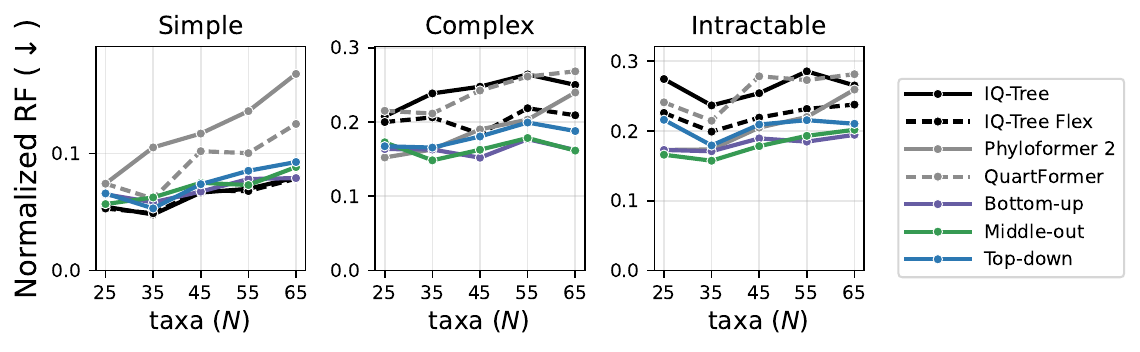}
        \caption{$L = 500$}
        \label{fig:rf_l}
    \end{subfigure}
    \caption{Our models make more accurate predictions than IQ-Tree for the complex and intractable evolutionary models.\textbf{(Left)} We hold $N$ fixed and vary the sequence length $L$. \textbf{(Right)} We hold $L$ fixed and vary the number of taxa $N$.
    The normalized RF distance of a random tree is roughly $1$.
    \vspace{-0.3cm}}
    \label{fig:rf}
\end{figure*}

\subsection{Our models accurately recover calibrated MCMC posteriors on real data}\label{sec: post calib}

We now take eight classical \textbf{real-world} evolutionary datasets from Treebase~\citep{Piel2000-bv}, a collection of inferred trees, originally selected in~\citet{Lakner2008-dd} (App.~\ref{app:real_data} for details).
These include diverse data from tetrapod ribosomal DNA to insect \textit{wingless} gene DNA, $N=27$ to $N=64$, and $L=378$ to $L=2520$, and they are numbered DS1-8 in order of increasing $N$.
We also simulate JC data with the same $N$ and $L$ to each of these datasets.
For each dataset, we infer $1000$ trees using our JC models, and compare the distribution of our inferred trees to the true posterior estimated by MrBayes under the JC model.
We remove constant sites in the sequence and shuffled the order of sites before inferring the tree.

In Fig.~\ref{fig:mmd} we report MMD distance between the posteriors of each model and the gold standard posterior samples from MrBayes for real and synthetic data; we also compare the Spearman correlation of the posterior probabilities of tree partitions and ``coverage'' of high probability partitions.
See App.~\ref{app: posterior calibration} for full results, tree samples, and an explanation of these metrics.
Our models accurately recover the posterior, obtaining much closer posteriors than the previous state-of-the-art Phyloformer 2.
However performance is superior on synthetic data, indicating room for improvement.

\subsection{Our models successfully infer trees for complex evolutionary models}

For the CAT and InDel datasets, obtaining accurate MCMC posteriors is prohibitively expensive or impossible.
To make predictions in these settings, practitioners often resort to point estimates using workhorse maximum likelihood methods such as IQ-Tree.
Practitioners use these procedures despite knowing they are misspecified in hopes that the evidence still points to the correct tree.

As an additional baseline, we also train Quartformer models, adapting the released code from \citep{Su2025-xt}. See \cref{app:quartformer} for more details on our Quartformer implementation.

For each method we infer $8$ trees for every set of sequences $X_{1:N}$ and take the majority-rule unrooted consensus tree, allowing for polytomy.
We evaluate the unrooted Robinson-Foulds (RF) distance with respect to the true tree.
We show our results for $N=45$ and $L=500$ in Fig.~\ref{fig:rf}. 
We see all of our models outperform Phyloformer 2.
Our models also perform just as well as IQ-Tree on the JC data where it is well-specified.
On the CAT and InDel datasets, IQ-Tree ignores more complex effects leading to lower-quality predictions.
In contrast, our models are able to leverage their knowledge about complex evolutionary processes and indels to make better predictions than IQ-Tree.
This is also true on the complex data where IQ-Tree Flex in principle models the same sort of site-variability and non uniform mutation processes that we're including in our simulated data;
    the mismatch between the exact form of the complex data prior and the effects IQ-Tree Flex is modelling however allows our models to outperform.
App.~\ref{app: timing results} shows run-time comparisons, showing our \bu{} and \mo{} models are also faster than IQ-Tree for most $N, L$ settings.

These results demonstrate that our models can give practitioners higher-quality point estimates of the tree under realistic evolutionary models.
Moreover, in contrast with bespoke algorithms such as IQ-Tree, our models simply require the practitioner to provide simulated data.

\section{Conclusion}
\label{sec:conclusion}

We develop flexible deep generative models that can be trained to accurately model posteriors over phylogenetic trees. Our models are trained on simulated data only, allowing them to learn posteriors for more realistic evolutionary models where traditional Bayesian approaches are intractable. 
We found, of the three paradigms, \bu{} to be consistently most accurate, but all three are natural starting points for extensions, and future practitioners may find the mathematics of \mo{} or \td{} more amenable to their setting.
Our contributions open up promising directions for future work: building a universal default tree prior, building better models for inferring evolutionary parameters, and scaling to larger datasets.

\paragraph{A universal tree prior}
There are a plethora of evolutionary models proposed in the literature. In many cases, practitioners may not know which evolutionary model best explains their data, and resort to using a model that is a poor match for the data. For many protein and DNA evolution datasets, the \texttt{GTR+G4+I} prior in MrBayes~\citep{Huelsenbeck2001-mo} acts as a default, and many evolutionary parameters come with common and default values used nearly universally for inference~\citep{Susko2018-nz}. Our framework opens the door for a phylogenetic inference foundation model: by training on simulated data from a variety of evolutionary models, the model automatically selects the appropriate evolutionary process as part of the posterior inference, and can be used across a large variety of applications. The foundation model can also be fine-tuned on simulated data  when a particular evolutionary model is desired. We believe that phylogenetic inference is therefore a promising application area for graph foundation models. 
This strategy may also address our decrease in performance on realistic out-of-distribution data (App.~\ref{app: posterior calibration}).
\citet{Potapczynski2026-rr} for example saw that training on a wider variety of synthetic data allowed them to perform substantially better on real data.

\paragraph{Inferring evolutionary parameters}
In this paper we focused on inference of the tree topology $T$ alone. However, in many settings we are also interested in inferring other parameters, such as branch lengths. As our work addresses the combinatorially challenging part of the inference problem, it can be extended by applying off-the-shelf Bayesian inference methods for Euclidean data to infer other parameters~\citep{Wang2015-md, Papamakarios2016-tp, Fourment2024-tk}. 

\paragraph{Species trees and reticulations}
We have focused here on inferring the sequence evolution of individual genes.
However real evolutionary data has complex phenomena such as recombination or incomplete lineage sorting which require reticulated or species trees~\citep[Ch. 27, 28]{Humphries2013-oc, Felsenstein2004-xi}; 
    real sequence data, such as pandemic data, can also come with information about the time it was collected~\citep{Drummond2003-cy}.
Future work could extend our models to deal with these data as well.

\paragraph{Larger scale inference}
We train our models on datasets with up to $N=65$ sequences of length up to $L=2600$. 
These sizes are sufficient for most important phylogenetic tasks.
It is a misconception that large-data applications such as pandemic data or full-genome analyses necessarily involve the analysis of $N=$ thousands or $L=$ millions all at once.
Instead practitioners more often infer sub-trees for subsets of $N$ or individual genes $L\approx 1000$ and combine their inferred trees using super-tree methods~\citep{Warnow2018-bv, Jiang2024-yw}, or species-tree methods~\citep{Mirarab2014-fy}.
Our method can be combined with these, or scale much further by using methods that have enabled large-scale graph foundation models \citep{bechlerspeicher2026billionscalegraphfoundationmodels}.

\paragraph{Reducing uncertainty}
We saw in App.~\ref{app: posterior calibration} that our posterior predictions were biased to be under-confident, a hallmark of models trained to maximize likelihood.
For language models trained to maximize likelihood, the most popular strategy to get high quality samples is reinforcement learning~\citep{Ouyang2022-vz} or sampling strategies~\citep{Holtzman2019-rr}.
Unlike the directions in the rest of this section, adapting these strategies to phylogenetic inference will involve new algorithms, taking care to not to miscalibrate the uncertainty prediction.

\paragraph{Acknowledgements}
This work was supported by DARPA AIQ
HR00112590066, NSF CAREER IIS-2145492, and NSF CDS\&E-MSS 2134216.

\clearpage
\bibliography{bib}
\bibliographystyle{plainnat}

\clearpage
\onecolumn
\appendix
\crefalias{section}{appendix}
\setcounter{table}{0}
\renewcommand{\thetable}{A\arabic{table}}

\addcontentsline{toc}{section}{Appendix}
\part{Appendix}
\parttoc

\section{Extended related works}
\label{app: ext_rel_works}

Here we discuss the improvements our models make over two related works, ARTree and Phyloformer 2.
This extends the discussion in Sec.~\ref{sec:rel_work}.

\subsection{Improvements upon ARTree}
\label{app:art_limits}

ARTree \citep{Xie2023-ed} is a scalable generative model over trees previously used for variational inference and learning representations.
It generates trees in the \mo{} paradigm.

To our knowledge, ARTree has always applied on a fixed set of $N$ sequences a priori $X_1, \dots, X_N$.
Therefore it iteratively added them to the tree in order, that is, $X_1$ was always added first and $X_N$ always added last.
As well, since the sequences themselves $X_{1:N}$ were fixed during training, they never needed to be fed into a neural network.

In our situation we 1) have a \textbf{changing} set of sequences we have to use to output a posterior 2) have no canonical order $1, \dots, N$.
We make two changes to accommodate these differences.
\begin{enumerate}[noitemsep]
    \item \textbf{Conditioning on sequences:}
    Our network not only conditions on the state of the tree $T_m$, but also all of the sequences $X_{1:N}$.
    Conditioning on a further $>10^5$ tokens required new architectures to prevent poor training behaviour as we show in App.~\ref{app:architecture} and Fig~\ref{fig:graph ablation}.
    \item \textbf{Rao-Blackwellizing the addition order}:
    We could have assigned a canonical order to each set of sequences $X_{1:N}$ and added them to the tree in that order as in ARTree.
    However, this would make the training objective sparse, as there is only one valid addition at each generation step.
    In Fig.~\ref{fig:RB ablation} we demonstrate through ablations that integrating over all possible additions makes a large difference in practice.
\end{enumerate}

Our \mo{} can be thought of as an ARTree model with (1) an architecture that conditions on sequences and (2) integrates over the addition order.

\subsection{Improvements upon Phyloformer 2}
\label{app:pf_limits}

Phyloformer 2 \cite{Blassel2025-vq} is the previous state-of-the-art method for simulation-based inference for Bayesian phylogenetic inference.
Our work on our \bu{} paradigm resembles Phyloformer 2 but improves upon it in several ways:
\begin{enumerate}[noitemsep]
    \item \textbf{Flexibility:} Phyloformer 2 uses a neural network backbone they call EvoPF that predicts embeddings for each species, followed by a module they call BayesNJ which parameterizes the distribution over topologies conditioned on the embeddings. BayesNJ uses a symmetric bilinear layer to convert the embeddings into a probability distribution over each merge, using the same embeddings for the full generation process. However, as noted in Section 3.2 of \cite{Blassel2025-vq}, in general the posterior merge probabilities at a given step might depend on the previous merges.
    We prove that there are simple distributions which cannot be approximated by Phyloformer 2 in Prop.~\ref{prop: pf2 not flex}.
    Our models unlock this expressivity by using iterative forward passes through the network, each conditioning on the current state of the tree.
    \item \textbf{Reliance on canonical merge order}: Another challenge faced by Phyloformer 2 is that it is intractable to sum over the likelihoods of all possible merge orders consistent with a given topology $T_0$.
    To solve this, they introduce a canonical merge order that depends on the branch lengths of the tree.
    As a result, the training objective is sparse, as there is only one valid merge at each generation step.
\end{enumerate}

In App.~\ref{app:pf2 ablations} we demonstrate through ablations that indeed our \bu{} model improves upon Phyloformer 2 by being more flexible and integrating over merge orders.

\subsubsection{Proof of inflexibility of Phyloformer 2}

We show that there are simple distributions over trees that cannot be represented by BayesNJ, the distribution family used by PF2.
We can parameterize the BayesNJ distribution with a set of representations $v_1, \dots, v_N$, a function that updates representations $g$, and a function that takes two representations and outputs a logit of their merger $f$.
At each step, we calculate the logits of all possible merges $(f(v_i, v_j))_{1=i<j=N}$, samples from $i, j\sim\mathrm{Categorical}(\mathrm{softmax}(f(v_i, v_j))_{1=i<j=N})$, merge nodes $i, j$ and create a new representation for the merged node $v_{(i, j)}=g(v_i, v_j)$.
In practice $v_{1:N}$ are 128-dimensional vectors and $g$ and $f$ are bilinear transformations, but we do not need this fact.

The idea for our example is that dependency between the clades cannot be represented by BayesNJ since representations in one clade are not affected by merges in another.
A more expressive merge function or representation-to-logit map cannot fix this. 
And dependencies can occur in realistic data when the hyperparameters in one clade influence those in another.

For 6 species $A, B, C, D, E, F$, call $T^1=((A, B), C), T^2=((A, C), B)$, $T^{\prime 1}=((D, E), F), T^{\prime 2}=((D, F), E)$, and use $*$ to represent the fusion of two trees, so $T^1*T^2=(((A, B), C), ((D, E), F)).$
\begin{proposition}
    Say we have a BayesNJ distribution $p$ where $p(T)<\epsilon$ for any tree $T$ where $A, B, C$ are not separated from $D, E, F$.
    Then $|p(T^1*T^{\prime 1})p(T^2*T^{\prime 2}) - p(T^1*T^{\prime 2})p(T^2*T^{\prime 1})|<K\epsilon$ for an absolute $K>0$.
    Therefore BayesNJ cannot approach a distribution which puts mass on just trees $T^1*T^{\prime 1}$ and $T^2*T^{\prime 2}$ as the first term stays away from $0$ and the second goes to $0$ in the bound.
\end{proposition}

\begin{proof}\label{prop: pf2 not flex}
    BayesNJ contructs a tree in a series of merge steps, just like bottom up, so we get a distribution $p(T_0, T_1, ..., T_6).$
    Let $n_1$ be the first step where a node in $A, B, C$ merges with something else.
    Then if $\mathrm{merge}(A, B, T_{n_1-1})$ is $T_{n_1-1}$ with $A$ and $B$ merged,
    $$p(T_{n_1}=\mathrm{merge}(A, B, T_{n_1-1})\mid T_{n_1-1})=\frac{\exp(f(v_A, v_B))}{\sum_{i<j\in\{A, B, C\}}\exp(f(v_i, v_j))}+O(\epsilon),$$
    with a similar equality for the merger of $A$ and $C$.
    Note this doesn't depend on $T_{n_1-1}$ beyond a factor of $O(\epsilon).$
    Calling $n_2$ the the first step where a node in $D, E, F$ merges we get a similar result.
    So,
    \begin{equation*}
        \begin{aligned}
            p(T^1*T^{\prime 1}) = & p( T_{n_1}=\mathrm{merge}(A, B, T_{n_1-1}), T_{n_2}=\mathrm{merge}(D, E, T_{n_2-1})) + O(\epsilon)\\
            =& \frac{\exp(f(v_A, v_B))}{\sum_{i<j\in\{A, B, C\}}\exp(f(v_i, v_j))}\times\frac{\exp(f(v_D, v_E))}{\sum_{i<j\in\{D, E, F\}}\exp(f(v_i, v_j))}+O(\epsilon).
        \end{aligned}
    \end{equation*}
    Similar equalities hold for $T^1*T^{\prime 2}$, $T^2*T^{\prime 1}$, and $T^2*T^{\prime 2}$.
\end{proof}

\subsubsection{Correcting leakage bug in Phyloformer 2 code base}

In \citet{Blassel2025-vq}, Phyloformer 2 is shown to outperform IQ-Tree on both tractable and intractable models of evolution. We have identified that this result is tainted by test leakage in their pipeline.

Specifically, their EvoPF architecture is not invariant to the permutation of sequence labels. This is because $\text{flatten}(a_i \otimes a_j) \neq \text{flatten}(a_j \otimes a_i)$ in the outer product mean layer. To generate input datasets, they use AliSim, which outputs rows in a traversal order which contains information about the tree topology. While we make sure to permute sequence labels before passing them into any model, reflecting a real-world application where there is no known ordering, \citet{Blassel2025-vq} skip this step. As such, their model learns to leverage the ordering of the leaves in predicting the tree topology.

We loaded the Phyloformer 2 checkpoint from the latest code release and evaluated normalized RF on 100 greedily-sampled trees from the LG+GC test data at \url{https://zenodo.org/records/13742527/files/paper_test_sets.tar.xz?download=1} with $N=50, L=500$. Without shuffling the input sequences, we obtain a score of 0.046, close to the reported value of around 0.048 in \citet{Blassel2025-vq}. However, after shuffling the input sequences the normalized RF jumps from 0.046 to 0.238.

We then trained our models and Phyloformer 2 on $N=50, L=500$ trees generated from their LG+GC simulated data, but with shuffled sequence order. We trained all models for the same wall time, 28 hours on 2 A100 GPUs, and evaluated on 100 8-consensus trees. IQ-Tree LG+G8 scores a normalized RF of 0.064. Our models scores 0.069, 0.071, and 0.083 for \bu{}, \mo{}, and \td{} respectively. Our implementation of Phyloformer 2 scores slightly worse, at 0.087. Notably, after fixing the training pipeline to eliminate leakage, Phyloformer 2 does not outperform IQ-Tree which is nearly correctly specified.

\section{Other potential paradigms for generative models over trees}
\label{app: other tree models}

In our work, we consider three natural paradigms for generative models over trees: \bu, \mo, and \td. 
Here, we explore other possible paradigms that could be used to define generative models over trees or discuss the obstacles of using them for flexible simulation-based inference.
The purpose of this appendix is to discuss the breadth of our three paradigms in the context of the wider mathematical phylogenetics literature.

\begin{table*}[h]
\centering
\caption{Summary of alternative generative modeling approaches for trees and their limitations.}
\label{tab:tree-models}
\begin{tabular}{p{6.5cm} p{5cm}}
\toprule
\textbf{Approach} & \textbf{Status or challenge} \\
\midrule
Splits and edge contractions & analogous to \td \\
\addlinespace
Leaf delete-regraft & analogous to \mo \\
\addlinespace
Add one leaf at a time and quarter puzzling & analogous to \mo \\
\addlinespace
Coalescents & analogous to \bu \\
\addlinespace
Dyck paths and permutations & models planar, unlabeled trees \\
\addlinespace
Hyperbolic embeddings  & models planar trees \\
\addlinespace
NNI, SPR, and TBR moves & tractability of conditionals \\
\addlinespace
Tree matchings & tractability of conditionals \\
\addlinespace
BHV space  & tractability of conditionals \\
\bottomrule
\end{tabular}
\end{table*}

\subsection{Paradigms that recover our models}

The following paradigms are practical for simulation-based inference, and can fit into one of the three paradigms which we consider.

\paragraph{Splits and edge contractions}
Many distributions of trees start with all species starting at the root, and then splitting left or right at a certain rate~\citep{Aldous2001-ec, Ford2005-gf};
these can include birth-death models.
Indeed this was explored for building generative models in~\citet{Larget2013-kx} and~\citet{Zhang2018-fv}.
However, the number of possible splits of $n$ species is $2^{n-1}$; it is therefore a challenge to parameterize a distribution over splits in a flexible way.
Indeed~\citet{Larget2013-kx} and \citet{Zhang2018-fv} had to restrict the possible splits to those seen in an MCMC run.
Similarly, one could imagine corrupting trees by contracting random edges rather than just leaf edges;
    the de-corruption process would need to deal with the same challenge of parameterizing a distribution of splits of objects.
One could parameterize such a distribution by gradually having nodes split one way or another using the diffusion method of~\citet{Ludke2024-zr}. This recovers the \td{} paradigm we explore.

\paragraph{Leaf delete-regraft}
The ``Aldous chain'' involves tree corruptions via pruning and regrafting individual leaves~\citep{Aldous2000-jv}.
Extensive theory on the process relies on the theory of random Dyck paths~\citep{Lohr2018-lk}, which we discuss below.
We are not sure if $p(T_m \mid T_0)$ is tractable, however we note that schedule-conditioning this model as described in~\citet{Amin2025-ag} leads to something analogous to \mo{}.
In particular, \citet{Amin2025-ag} advocates improving discrete diffusion by keeping track of how many times each token in a sequence has been corrupted during diffusion.
But once a leaf has been corrupted, it no longer contains any information about its original position.
Therefore, the task of predicting $q_\theta(T_{m-1} \mid T_m, \vec s_M)$, where $\vec s_M$ is an indicator of how many times each leaf has been corrupted, simply involves taking the subtree of un-noised leaves and predicting which edge any noised leaf falls in that tree, recovering the exact prediction problem from \mo.

\paragraph{Add one leaf at a time and quartet puzzling}
\citet{Sainudiin2015-jo} and~\citet{Ford2005-gf} describe distributions of trees in which one leaf is added at a time; these differ by introducing different hyperparameters and latent variables, and possibly disallowing addition to edges not connected to a leaf.
Integer vector encodings of trees describe the order in which individual leaves are attached~\citep{Richman2025-no}.
Quartet puzzling is also an algorithm that adds one leaf at a time to a tree based on likelihoods of ``quartet'' splits~\citep{Strimmer1996-hs}. These approaches are analogous to the \mo{} paradigm we explore.

\paragraph{Coalescents}
One of the most popular models for for phylogenetic trees is the coalescent~\citep{Bertoin2010-bh}.
These involve joining subtrees at some rate, which is equivalent to the \bu{} paradigm we explore.

\subsection{Paradigms that model the wrong types of trees}

Alternatively, some paradigms are not obviously applicable for our problem since they model the wrong types of trees. We are interested in modeling \textit{non-planar} trees (which do not distinguish between left and right branches) and \textit{labeled} trees (which distinguish the leaves).

\paragraph{Dyck paths and permutations}
Every labeled planar tree -- that is, a tree distinguishing left and right children -- can be uniquely described by a Dyck path and a permutation of $\{1, \dots, N\}$, suggesting we could model trees over these two objects~\citep{Aldous2000-jv}.
The coalescent point process~\citep{Lambert2013-xw} is another clever approach to efficiently generating trees; every labeled planar tree with branch lengths is characterized by a vector in $(0, 1)^{N-1}$ and a permutation. One issue with this approach is that it is not clear that modeling permutations or Dyck paths is simpler than modeling trees.
Another issue is that we wish to model non-planar trees, and each non-planar tree admits $2^{N-1}$ planar representations.
Stochastically choosing a planar representation for each tree would significantly increase variance when training, but canonically choosing a planar representation would introduce discontinuities or artificial structure.

\paragraph{Hyperbolic embeddings}
\citet{Sarkar2012-hm} shows how to embed trees into hyperbolic space with arbitrarily small distortion.
Putting the distortion aside, the embedding is not unique: rotations, translations, and the ordering of cones emanating from each internal node do not affect the tree distance matrix.
In particular, the ordering of cones can be used to define a planarity of the tree, showing that the symmetries of this embedding cause problems at least as bad as those of the Dyck paths and permutations above.

\subsection{Paradigms that do not admit tractable conditionals for discrete diffusion training}

The following paradigms are challenging choices for simulation-based inference since they do not allow for efficient training with a discrete diffusion objective as explained in Sec.~\ref{sec:diffusion}. Specifically, to train a diffusion model as in~\citet{Campbell2022-zm}, we need the conditionals $p(T(t)\mid T(0))$ to be tractable for any time $t$, or to train a diffusion model as in~\citet{Amin2025-ag}, we need $p(T_m \mid T_0)$ to be tractable for any number $M$, where $T_m$ is the tree after $m$ corruptions\footnote{SCUD~\citep{Amin2025-ag} models can technically be trained with only empirical samples from $p(T_m \mid T_{m-1}, T_0)$. However this training method sacrifices a closed-form ELBO and has extremely high variance in its gradient estimates~\citep{Baron2026-mj}.}.
Many stochastic processes on trees are not suitable for a generative model because these conditionals are intractable.

\paragraph{NNI, SPR, and TBR moves}
Many MCMC methods explore tree-space with proposals based on nearest neighbor interchange (NNI), subtree prune-and-regraft (SPR), and tree bisection and reconnection (TBR) moves.
However, calculating whether two trees $T_1, T_2$ have NNI, SPR, or TBR distance less than or equal to $M$ is NP-Hard~\citep{DasGupta2000-vp, Bordewich2005-bt, Allen2001-so}.
Therefore, if we corrupt our trees using these moves, simply determining if $p(T_m = T' \mid T_0 = T)>0$ is NP-Hard, making it likely intractable.
Meanwhile, other tree distances which are easier to calculate, like RF distance, are not associated with any particular sequence of tree operations to our knowledge.

\paragraph{Tree matchings}
\citet{Diaconis2002-ek} showed that trees are in bijection with pairings of the numbers $\{1, \dots, 2N-2\}$.
They characterized the eigenspectrum of the action of the symmetric group on this representation, finding deep algebraic structure.
However, the exact formulas for calculating $q_\theta(T_m \mid T_0)$, while ``solved'', are not easily computable in practice.

\paragraph{BHV space}
On the continuous Billera-Holmes-Vogtmann tree space~\citep{Billera2001-av}, where topologies change via NNI moves, the remarkable theorem of~\citet{Owen2011-ny} demonstrates that geodesic distances are still computable in polynomial time.
Nevertheless, recent work has highlighted the intractability of calculating the densities of diffusion processes on this space~\citep{Woodman2025-vc}; tractable distributions meanwhile (GGF of ~\citet{Woodman2025-vc}) have restricted support.
It is not clear if other continuous parameterizations, like that of~\citet{Gavryushkin2016-ox}, have more tractable diffusions.

\section{Expressivity of sampling paradigms}
\label{app:expressivity}

Here we prove simple modifications to the logits can allow all three paradigms to generate trees from PDA to Yule, the spectrum of realistic trees.
Note this section is not about proving that our models can represent every sequence distribution (that fact follows from the discussion below Eqn.~\ref{eqn:elbo}).
Rather it is about seeing how each of our paradigms fit simple distributions of trees explicitly.

\subsection{\Bu{}}

Consider a process that generates a tree $T_0$ by starting from $T_{N-1}$, a forest of $N$ singleton components, and building up to a complete tree.
At each step, the process merges two components of the forest $T_m$ rooted at $v$ and $u$ leaves to produce $T_{m-1}$, selecting this pair with probability proportional to $\gamma(n_{\mathrm{leaves}}(v) + n_{\mathrm{leaves}}(u) - 2) + 1$, where $\#\mathrm{leaves}(u)$ is the number of leaves in the subtree rooted at $u$, for some parameter $\gamma \geq 0$.
Each merge creates one internal node of the final tree $T_0$; the first merge creates the deepest internal node and the last creates the root.

Reversing the merge order, we obtain an ordering $v_1, v_2, \dots, v_{N-1}$ of the internal nodes of $T_0$ in which every ancestor appears before its descendants --- that is, a linear extension of $\texttt{int}(T_0)$.
It will be useful below that paths are uniform over linear extensions.

\begin{proposition}\label{prop:gamma_coalescent}
    When $\gamma = 0$, $p(T_0)$ is a Yule distribution over trees; when $\gamma \to \infty$, $p(T_0)$ concentrates on combs.
    When $\gamma = 1$, $p(T_0) = 1/(2N-3)!!$ is the PDA (uniform) distribution over topologies, and each path is equally likely conditioned on the topology:
    $p(T_1, \dots, T_{N-1}\mid T_0)=1/|\mathrm{LE}(\texttt{int}(T_0))|.$
\end{proposition}

\begin{proof}
The $\gamma=0$ case is Kingman's coalescent, which is well known to produce the Yule distribution~\citep{Aldous2001-ec}; the $\gamma\to\infty$ case follows since the merge probability concentrates on the pair with largest $a+b$, which always includes the component containing all previously merged taxa.

For $\gamma=1$, at step $m$, the forest $T_m$ has $m+1$ components with leaf counts $n_1, \dots, n_{m+1}$.
The normalizing constant is
$$\sum_{i<j}(n_i+n_j-1) = ((m+1)-1)\sum_i n_i - \binom{m+1}{2} = \frac{1}{2}m(2N-(m+1)),$$
since each $n_i$ appears in $(m+1)-1$ pairs.
Taking the product from $m=1$ to $m=N-1$, we get
$$\prod_{m=1}^{N-1}\frac{1}{2}m(2N-(m+1)) = 2^{-(N-1)}(N-1)!\frac{(2N-2)!}{(N-1)!}=(N-1)!(2N-3)!!,$$
with the last equality coming by counting even and odd terms of $(2N-2)!$.
We also note that for $v\in\texttt{int}(T_0)$, $n_{\mathrm{leaves}}(\texttt{left}(v_m))+n_{\mathrm{leaves}}(\texttt{right}(v_m))$ is the number of decendants $v$ has in $\texttt{int}(T_0)$ plus two: $|\texttt{descendants}_{\texttt{int}(T_0)}(v)|+2$.
Finally,
\begin{equation*}
\begin{aligned}
    p(T_0, T_1, \dots, T_{N-1})=&\prod_{m=1}^{N-1}\frac{n_{\mathrm{leaves}}(\texttt{left}(v_m))+n_{\mathrm{leaves}}(\texttt{right}(v_m))-1}{\frac{1}{2}m(2N-(m+1))}\\
    =&\frac{\prod_{v\in\texttt{int}(T_0)}(|\texttt{descendants}_{\texttt{int}(T_0)}(v)|+1)}{(N-1)!(2N-3)!!}
\end{aligned}
\end{equation*}
which does not depend on the order of the merges.
This then equals $1/(|\mathrm{LE}(\texttt{int}(T_0))|(2N-3)!!)$ by the hook length formula $|\mathrm{LE}(\texttt{int}(T_0))|=|\texttt{int}(T_0)|!/\prod_{v\in \texttt{int}(T_0)}(|\texttt{descendants}_{\texttt{int}(T_0)}(v)|+1)$.
Summing over all $|\mathrm{LE}(\texttt{int}(T_0))|$ linear extensions gives $p(T_0)=1/(2N-3)!!$.
\end{proof}

\subsection{\Mo{}}

Consider a process that generates a rooted tree by starting with a tree with a single pseudo taxon and grafting one taxon at a time.
At step $k$, starting at $k=0$, the current tree has $k$ leaves and $\max{}(2k-1, 1)$ edges.
The new taxon subdivides an edge chosen with probability proportional to
$\alpha$ or $1-\alpha$, where \emph{pendant} edges (those incident to a leaf, not including the pseudo-taxon) receive
weight $1-\alpha$ and all other edges receive weight~$\alpha$,
for a parameter $\alpha\in[0,1)$.

\begin{proposition}\label{prop:mo_expressivity}
When $\alpha=0$, $p(T_0)$ is a Yule distribution;
when $\alpha=\tfrac{1}{2}$, $p(T_0)$ is PDA;
and when $\alpha\to 1$, $p(T_0)$ concentrates on combs.
When $\alpha=1/2$, the distribution does not depend on the order in which
taxa are inserted.
\end{proposition}

\begin{proof}
This is Ford's alpha model~\citep{Ford2005-gf}.
To correspondence to Yule, PDA, and combs is given in \citep[Sec 3.12]{Ford2005-gf}.
Order-independence when $\alpha=1/2$ follows from the uniform probability of adding to any edge:
\[p(T_{0:N})=\frac 1 {N!}1\times 1\times 3^{-1}\times 5^{-1}\times\cdots\times(2N-3)^{-1}=\frac 1 {N!}p_\mathrm{PDA}(T).\qedhere\]
\end{proof}

\subsection{\Td{}}
\label{app:td_expressivity}

The Aldous beta model~\citep{Aldous_beta} generates a labeled
rooted tree by recursively splitting sets of taxa.
At a node containing $n\geq 2$ taxa $v_1, \dots, v_n$, the set is partitioned into two
non-empty groups; the probability of each labeled partition $\{A,B\}$ is
proportional to $B(|A|+\beta+1,\;|B|+\beta+1)$
for $\beta > -2$, where $B(\cdot, \cdot)$ is the beta function.
Each group is then split recursively until every group is a singleton.
When $\beta=0$ this gives the Yule distribution; when $\beta=-\tfrac{3}{2}$
it gives PDA; and as $\beta\to\infty$ it gives the ``almost completely
balanced'' distribution~\citep{Aldous2001-ec}.

In our top-down paradigm, we do not split all taxa at a node simultaneously.
Instead, taxa percolate down one at a time and all taxa at a given depth split into groups before any taxon proceeds to the next depth.
We imagine doing this by picking a random taxa and ``revealing'' how it split under the Aldous beta model.
We can reveal the direction of one taxa at a time by looking at the marginal likelihood: if we've revealed that $a$ taxa, call them $v_1, \dots, v_a$, have percolated into one group and $b$ taxa, call them $v_{n-b+1}, \dots, v_n$, have percolated into the other, what is the posterior probability that the next taxa joins the group with $a$ taxa?
We call this probability with $M=n-a-b$ taxa remaining
$$f(a,b\mid M)=\frac{p(v_{1:a+1}\in A,
    v_{n-b+1:n}\in B)}{p(v_{1:a}\in A,
    v_{n-b+1:n}\in B)}.$$
Note the order of $A$ and $B$ doesn't matter as $f(a, b\mid M)=1-f(b, a\mid M)$, and, when $a=b=0$, there is no distinction between $A$ and $B$.
The procedure is summarized in Alg.~\ref{alg:td-generate}.
\begin{algorithm}[H]
\caption{Single-step Aldous $\beta$}
\label{alg:td-generate}
\begin{algorithmic}
\STATE Initialize $T$ with a single root with $\texttt{leaves}(\texttt{root}) = \{v_1, \dots, v_N\}$
\STATE $d \gets 0$
\WHILE{some node at depth $d$ has $\geq 2$ taxa}
    \STATE $U_d \gets \{u \in T : d_T(u) = d,\ |\texttt{leaves}(u)| \geq 2\}$
    \WHILE{$\texttt{leaves}(u) \neq \emptyset$ for some $u \in U_d$}
        \STATE $u \sim \mathrm{Categorical}\left(|\texttt{leaves}(u)|/ \textstyle\sum_{u' \in U_d} |\texttt{leaves}(u')|\right)$
        \STATE $v \sim \mathrm{Uniform}(\texttt{leaves}(u))$
        \IF{$u$ has no children}
            \STATE create children $w, w'$ of $u$;\quad $\texttt{child} \gets w$
        \ELSE
            \STATE $w, w' \gets$ children of $u$;\quad
                   $a, b \gets |\texttt{leaves}(w)|, |\texttt{leaves}(w')|$
            \STATE $M \gets |\texttt{leaves}(u)|$
            \STATE $\texttt{child} \gets w$ with probability $f(a, b \mid M)$ else $\texttt{child} \gets w'$
        \ENDIF
        \STATE $\texttt{leaves}(\texttt{child}) \gets \texttt{leaves}(\texttt{child}) \cup \{v\}$
        \STATE $\texttt{leaves}(u) \gets \texttt{leaves}(u) \setminus \{v\}$
    \ENDWHILE
    \STATE $d \gets d + 1$
\ENDWHILE
\STATE \textbf{return} $T$
\end{algorithmic}
\end{algorithm}

To implement Alg.~\ref{alg:td-generate} we only need to know the values of $f$.

\begin{proposition}\label{prop:td_expressivity}
The probabilities~\eqref{eq:td_f_ratio} equal
\begin{equation}\label{eq:td_f}
    f(a,b\mid M)=\begin{cases}
      \dfrac{a+\beta+1}{a+b+2\beta+2},
        &a,b>0,\\[10pt]
      \dfrac{B(a+\beta+2,\beta+1)-B(n+\beta+1,\beta+1)}
            {B(a+\beta+1,\beta+1)-B(n+\beta+1,\beta+1)},
        &b=0,\;M>1,\\[10pt]
      0,
        &b=0,\;M=1,
    \end{cases}
\end{equation}
with the $a=0$ cases given by symmetry ($f(0,b)=1-f(b,0)$).
Every percolation order of the taxa is equally likely
conditioned on $T_0$.
\end{proposition}

\begin{proof}

For a node with $n$ taxa where $a$ have percolated one direction and $b$ the other,
with $M=n-a-b$ remaining, define the unnormalized marginal likelihood
\begin{equation}\label{eq:td_S}
    p(v_{1:a}\in A,
    v_{n-b+1:n}\in B)\propto S(a,b,M):=\sum_{\substack{\{A,B\}:|A|\geq 1,|B|\geq 1\\
    v_{1:a}\in A,
    v_{n-b+1:n}\in B}}
    B(|A|+\beta+1,|B|+\beta+1).
\end{equation}
Since $|A|=a+j$ and $|B|=b+M-j$ for $j$ of the
remaining taxa joining~$A$, we can write
\begin{equation}\label{eq:td_S_sum}
    S(a,b,M)=\sum_{j}
    \binom{M}{j}B(a+j+\beta+1,\;b+M-j+\beta+1),
\end{equation}
where the sum runs over all $j$ for which $a+j\ge 1$ and $b+M-j\ge 1$.
The conditional probability that the next taxon goes to group with $a$ taxa is therefore
\begin{equation}\label{eq:td_f_ratio}
    f(a,b\mid M)=\frac{S(a+1,b,M-1)}{S(a,b,M)}.
\end{equation}
We evaluate $S(a,b,M)$ in closed form using the integral representation
$B(s,t)=\int_0^1 x^{s-1}(1-x)^{t-1}dx$ and the binomial theorem.

\paragraph{Case with $a,b>0$.}
When $a\geq 1$ and $b\geq 1$ the constraints $a+j\geq 1$ and $b+M-j\geq 1$ hold
for all $j\in\{0,\dots,M\}$, so from~\eqref{eq:td_S_sum}
\begin{align*}
    S(a,b,M)
    &=\int_0^1 x^{a+\beta}(1-x)^{b+\beta}
      \underbrace{\sum_{j=0}^M\binom{M}{j}x^j(1-x)^{M-j}}_{=1}dx
    =B(a+\beta+1,b+\beta+1).
\end{align*}
Therefore $S(a+1,b,M-1)=B(a+\beta+2,b+\beta+1)$ and
\[
    f(a,b)
    =\frac{B(a+\beta+2,b+\beta+1)}{B(a+\beta+1,b+\beta+1)}
    =\frac{a+\beta+1}{a+b+2\beta+2}.
\]

\paragraph{Case with $b=0$, $M\geq 1$.}
The constraint $b+M-j\geq 1$ forces $j\leq M-1$, excluding the
degenerate all-left split.  Writing $n=a+M$,
\begin{align*}
    S(a,0,M)
    &=\int_0^1 x^{a+\beta}(1-x)^{\beta}
      \left[\underbrace{\sum_{j=0}^M\binom{M}{j}x^j(1-x)^{M-j}}_{=1}
      -x^M\right]dx\\
    &=B(a+\beta+1,\beta+1)-B(n+\beta+1,\beta+1).
\end{align*}
Similarly $S(a+1,0,M-1)=B(a+\beta+2,\beta+1)-B(n+\beta+1,\beta+1)$.
Dividing gives the $b=0$ case of~\eqref{eq:td_f}.
Note when $M=1$ we have $a+1=n$ so $S(a+1,0,0)=B(n+\beta+1,\beta+1)-B(n+\beta+1,\beta+1)=0$.

\paragraph{Order independence.}
In the notation of Alg.~\ref{alg:td-generate}, calling $M_d=\sum_{u\in U_d}n_{\texttt{leaves}}(u)$, and $a_u, b_u$ the number of leaves in $u$'s children for every node $u$,
it is easy to see that, by construction, the probability of the path is the probabilities of the percolation orderings times the probabilities from Aldous' beta model:
$$p(T_{0:M})=\prod_{d=0}^{d_{\max{}}} \frac{1}{M_d!}\prod_{u\in U_d}\frac{B(a_u+\beta+1, b_u+\beta+1)}{B(\beta+1, \beta+1)-2B(a_u+b_u+\beta+1, \beta+1)}=\left(\prod_{d=0}^{d_{\max{}}} \frac{1}{M_d!}\right)\times p(T_0).$$
The first term $\prod_{d=0}^{d_{\max{}}-1} \frac{1}{M_d!}$ is $p(T_{1:M}\mid T_0)$ which is therefore independent of the order of the percolation.
\end{proof}
\section{Discrete diffusion models}
\label{app:diffusion}

\subsection{Background}

Discrete diffusion models have been recently highlighted for flexibly fitting complex distributions of discrete objects.
However previous diffusion models have largely been built to model data of the form $\mathrm{tokens}^\mathrm{length}$ like language~\citep{Austin2021-dg, Campbell2022-zm, Shi2024-fs}.
To apply these models to graphs therefore, one could simply model their adjacency matrices $A\in\{0, 1\}^{N\times N}$~\citep{Vignac2022-du, Bohde2025-ok}.
Modeling phylogenetic trees as abstract graphs is problematic however as (1) internal nodes of the tree are latent, so there is no ground truth adjacency matrix $A$, and (2) trees are highly constrained structures, and there is no clear way to incorporate these constraints during generation.
Our models avoid these issues by building a diffusion model directly on the space of trees.

Discrete diffusion models have also succeeded at modeling other sorts of discrete objects, like permutation processes~\citep{Zhang2024-en}, and sequences of different lengths~\citep{Havasi2025-yx, Baron2026-mj}.
However, these methods do not map to modeling trees.

\subsection{Derivation of discrete diffusion ELBO}
\label{app:elbo}

\[
q_\theta(T_0 \mid X_{1:N}) =
\sum_{T_1, T_2, \ldots, T_M} \prod_{m=1}^M q_\theta(T_{m-1} \mid T_m, X_{1:N})
\]

Here, we provide the derivation of the discrete diffusion ELBO in Eqn.~\ref{eqn:elbo}. We start by taking the logarithm of the expression for the likelihood in Eqn.~\ref{eqn:likelihood}. Then, we multiply and divide each term by $p\left(T_{1:M} \mid T_0 \right)$, yielding

\begin{align*}
\log q_\theta(T_0 \mid X_{1:N})
&=
\log \sum_{T_{1:M}} \, p\left(T_{1:M} \mid T_0 \right) \frac{\prod_{m=1}^M q_\theta(T_{m-1} \mid T_m, X_{1:N})}{p\left(T_{1:M} \mid T_0 \right)} \\
&= \log \mathbb{E}_{T_{1:M} \sim p\left(T_{1:M} \mid T_0 \right)} \left[\prod_{m=1}^M \frac{q_\theta(T_{m-1} \mid T_m, X_{1:N})}{p(T_{m-1} \mid T_m, T_0)} \right] \\
&\geq
\mathbb{E}_{T_{1:M} \sim p\left(T_{1:M} \mid T_0 \right)} \left[\sum_{m=1}^M \log \frac{q_\theta(T_{m-1} \mid T_m, X_{1:N})}{p(T_{m-1} \mid T_m, T_0)} \right] \\
&=
\sum_{m=1}^M \mathbb{E}_{T_{m-1}, T_m \sim p\left(T_{m-1}, T_m \mid T_0\right)} \left[\log \frac{q_\theta(T_{m-1} \mid T_m, X_{1:N})}{p(T_{m-1} \mid T_m, T_0)} \right] \\
&=
- \sum_{m=1}^M \mathbb{E}_{T_m \sim p\left(T_m \mid T_0 \right)} \left[ D_{\text{KL}}(p(T_{m-1} \mid T_m, T_0) \, || \, q_\theta(T_{m-1} \mid T_m, X_{1:N})) \right] \\
&=- M\mathbb E_{m\sim\mathrm{Unif}(1, \dots, M), T_m \sim p\left(T_m \mid T_0 \right)} \left[ D_\textrm{KL}(p(T_{m-1} \mid T_m, T_0) \, || \, q_\theta(T_{m-1} \mid T_m, X_{1:N})\right].
\end{align*}
We used the fact that $p(T_M\mid T_0)$ and $q(T_M\mid X_{1:N})$ are both assumed to have all their mass on a unique starting position;
    for \td{} this has all the taxa at the root,
    for \bu{} this has all taxa disconnected,
    and for \mo{} the tree has only the pseudo-taxon.

\section{Additional information on paradigms}

For each paradigm we describe the noising schedule, how to calculate the target $p(T_{m-1}\mid T_m, T_0)$, and how to sample from our mutual-information informed noising schedule.

\subsection{Noise schedule details}
\label{app:noise_schedule}

The ELBO in Eqn.~\ref{eqn:elbo} requires computing
$p(T_{m-1} \mid T_m, T_0)$, the distribution over the
previous state given the current noised state and the
ground truth.
We achieve this by introducing a noise schedule $\pi(T_m \mid T_0)$ and rewriting the ELBO as $\mathbb E_{p(T_0, X_{1:N})}\log q_\theta(T_0 \mid X_{1:N})\geq$
\begin{align}
- \mathbb{E}_{T_m \sim \pi\left(\cdot \mid T_0 \right)} \left[ w(T_m) \cdot D_\textrm{KL}(p(T_{m-1} \mid T_m, T_0) \, \| \, q_\theta(T_{m-1} \mid T_m, X_{1:N}))\right]
\end{align}
where
\begin{equation}
w(T_m) = \frac{M \cdot p(T_m \mid T_0)}{\pi(T_m \mid T_0)}.
\end{equation}
A classic heuristic for designing discrete diffusion noise schedules is to choose $\pi$ to minimize the variance of this Monte Carlo estimator, assuming $q_\theta$ has converged to the true reverse process~\citep{Kingma2021-za, Austin2021-dg}.
To make this analytically tractable, we further assume a simple reference posterior $p(T_0 \mid X_{1:N})$.
A natural choice is any of the prior distributions described in Sec.~\ref{sec:expressivity} whose reverse process we can compute in closed form; we use $p_\mathrm{PDA}(T_0)$ as a concrete example throughout, though any such prior works equally well.

Under this assumption, the variance-minimizing $\pi$ sets $w( T_m) \cdot D_\mathrm{KL}(\cdots)$ to a constant across $(m, T_m)$, giving
\begin{equation}\label{eq:pi-optimal}
\pi(T_m \mid T_0) \propto M \cdot p(T_m \mid T_0) \cdot D_\textrm{KL}(p(T_{m-1} \mid T_m, T_0) \, \| \, p_\mathrm{PDA}(T_{m-1}\mid T_m)),
\end{equation}
where $p_\mathrm{PDA}(T_{m-1}\mid T_m)=\mathbb E_{T_0\sim p_{\mathrm{PDA}(T_0)}}p(T_{m-1} \mid T_m, T_0)$.
The normalizing constant of $\pi$ is
\begin{equation}\label{eq:zpi}
\sum_{m=1}^{M} \sum_{T_m} M \cdot p(T_m \mid T_0, m) \cdot D_\textrm{KL}(p(T_{m-1} \mid T_m, T_0) \, \| \, p_\mathrm{PDA}(T_{m-1}\mid T_m)) = -M \log p_\mathrm{PDA}(T_0),
\end{equation}
which is $-M\log(2N-3)!!$.
Under this choice,
$$w(T_m)=\frac{\log p_\mathrm{PDA}(T_0)}{D_\textrm{KL}(p(T_{m-1} \mid T_m, T_0) \, \| \, p_\mathrm{PDA}(T_{m-1}\mid T_m))}.$$

Marginalizing over $T_0 \sim p_\mathrm{PDA}$, the quantity $p(T_m \mid T_0) \cdot D_\textrm{KL}(p(T_{m-1} \mid T_m, T_0) \, \| \, p_\mathrm{PDA}(T_{m-1}\mid T_m))$ becomes the mutual information $\mathrm{MI}(T_{m-1};\, T_0 \mid T_m)$: the amount of information about $T_0$ that is lost at step $m$.
This mirrors the mutual information schedules of \citet{Austin2021-dg} and concentrates training on steps where the model must add the most information about $T_0$ in a single step.

To implement this schedules, we need two things:
\begin{itemize}
    \item The form of $p_\mathrm{PDA}(T_{m-1}\mid T_m)$ to calculate $w(T_m)$.
    \item A procedure to sample from $\pi(T_m \mid T_0)$.
\end{itemize}
Sampling will be handled below but note App.~\ref{app:expressivity} gives us the first goal:
    we've already built procedures $q$ to sample sequences of trees such that $\prod_{m=0}^{M-1} q(T_{m-1}\mid T_m)= p_{\mathrm{PDA}}(T_0)\times p_{\mathrm{noising}}(T_{1:M}\mid T_0)$.
    Therefore, these $p(T_{m-1}\mid T_m)$ must be $p_\mathrm{PDA}(T_{m-1}\mid T_m)$ as we will see below.

\subsection{\Bu{}}
\label{app:rb_bu}

\subsubsection{Distribution over paths}

It will be convenient to derive the distribution over paths $p(T_1, \dots, T_{N-1}\mid T_0)$.
The path $(T_m)_m$ is characterized by an order of internal vertices which are removed from $T_0$, $v_1, \dots, v_{N-1}$;
on the other hand every ordering of the internal vertices of $T_0$ $v_1, \dots, v_{N-1}$ corresponds to a unique path as long as the ancestors of every node appear before it in the ordering.
An ordering of the internal vertices that satisfies this rule is known in the literature as a \textit{linear extension} of the subtree of $T_0$ made of its internal nodes, \texttt{int}($T_0$).
We call $\mathrm{LE}(T)$ the set of linear extensions of a tree $T$.
We now show our choice of noising process makes \textit{every linear extension equally likely}; this is the property that will allow us to derive the quantities we need below.

\begin{proposition}
    The probability of a path $(T_m)_m$ that removes vertices in order $v_1, \dots, v_{N-1}$ is
    $$p(T_1, \dots, T_{N-1}\mid T_0)=1/|\mathrm{LE}(\texttt{int}(T_0))|.$$
\end{proposition}
\begin{proof}
Recall for a root of a subtree $v$,
$p(\texttt{remove}\ v\mid T_m)\propto|\texttt{descendants}(v)|,$
excluding $v$ as a descendant.
The denominator is the sum of the descendants of all roots, so it's simply the number of nodes in the tree, minus the roots: $2N-2(m+1)$.
Finally,
\begin{equation*}
\begin{aligned}
    p(T_1, \dots, T_{N-1}\mid T_0)=&\prod_{m=0}^{N-2}p(\texttt{remove}\ v\mid T_m)\\
    =&\prod_{m=0}^{N-2}\frac{|\texttt{descendants}(v_{m+1})|}{2N-2(m+1)}\\
    =&\frac{\prod_{v\in \texttt{int}(T_0)}|\texttt{descendants}(v)|}{2^{N-1}(N-1)!}
\end{aligned}
\end{equation*}
which does not depend on the order of the vertices.
One can verify this also matches the classical hook length formula by noting $|\texttt{descendants}(v)|$ is twice the number of descendants in $\texttt{int}(T_0)$ plus $1$, $|\texttt{descendants}_{\texttt{int}(T_0)}(v)|+1$, and that $|\mathrm{LE}(\texttt{int}(T_0))|=|\texttt{int}(T_0)|!/\prod_{v\in \texttt{int}(T_0)}(|\texttt{descendants}_{\texttt{int}(T_0)}(v)|+1)$.
\end{proof}

\subsubsection{Calculating the target}

We can now derive $p(T_{m-1}\mid T_m, T_0)$.
The set of removed nodes $\{v_1, \dots, v_m\}$ forms a rooted subtree of $\texttt{int}(T_0)$ containing the root; call it $R_m$.
Its leaves are exactly the candidates for the last node removed.
Since all linear extensions of $\texttt{int}(T_0)$ are equally likely, the probability that node $v$ was removed last is simply the fraction of linear extensions of $R_m$ that end in $v$:
$p(T_{m-1} = \texttt{add}(T_m,v)\mid T_m, T_0) = \frac{|\mathrm{LE}(\texttt{remove}(R_m, v)|}{|\mathrm{LE}(R_m)|}.$

\begin{proposition}
    For a leaf $v$ of $R_m$,
    \begin{equation}
        p(T_{m-1} = \texttt{add}(T_m,v)\mid T_m, T_0) = \frac{1}{m}\prod_{u\in\texttt{ancestors}(v, R_m)}\frac{|\texttt{descendants}_{R_m}(u)|+1}{|\texttt{descendants}_{R_m}(u)|},
        \label{eq:bu_reverse}
    \end{equation}
    where $\texttt{descendants}_{R_m}(u)$ counts the descendants of $u$ in $R_m$, not including $u$ itself.
\end{proposition}
\begin{proof}
    By the hook length formula for rooted trees,
    $|\mathrm{LE}(R_m)| = m!/\prod_{u\in R_m}(|\texttt{descendants}_{R_m}(u)|+1).$
    Removing a leaf $v$ from $R_m$ decreases the numerator from $m!$ to $(m-1)!$, removes the factor $|\texttt{descendants}_{R_m}(v)|+1=1$ from the denominator, and decrements $|\texttt{descendants}_{R_m}(u)|$ by $1$ for every strict ancestor $u$ of $v$.
    Taking the ratio gives Eqn.~\ref{eq:bu_reverse}.
\end{proof}
This can be calculated with a post-order traversal of $R_m$ to compute $|\texttt{descendants}_{R_m}(v)|$ for all $v$, then a pre-order pass to accumulate the product over ancestors.
The cost of computing $p(T_{m-1}\mid T_m, T_0)$ for every possible $T_{m-1}$ is therefore $O(m)$.

\subsubsection{Noise schedule}

\paragraph{The PDA backwards steps}

By Prop.~\ref{prop:gamma_coalescent}, the PDA coalescent ($\gamma = 1$) also produces uniform linear extensions conditioned on $T_0$.
Therefore the reverse model $p_\mathrm{PDA}(T_{m-1} \mid T_m)$ merges subtree roots $u, w$ in the forest $T_m$ with probability proportional to $n_\mathrm{leaves}(u) + n_\mathrm{leaves}(w) - 1$.

\paragraph{Sampling from $\pi$}

We write
\begin{equation*}
    \begin{aligned}
        \pi(T_m\mid T_0) \propto& p(T_m\mid T_0)D_\mathrm{KL}(p(T_{m-1}\mid T_m, T_0)\Vert p_\mathrm{PDA}(T_{m-1}\mid T_m))\\
        =& \sum_{\tilde T_{m-1}}p(T_m, \tilde T_{m-1}\mid T_0)\log \frac{p(\tilde T_{m-1}\mid  T_m, T_0)}{p_\mathrm{PDA}(\tilde T_{m-1}\mid  T_m)}\\
        =&\sum_{\tilde T_{1:M}\mid \tilde T_m=T_m}p(\tilde T_{1:M}\mid T_0)\log \frac{p(\tilde T_{m-1}\mid \tilde T_m, T_0)}{p_\mathrm{PDA}(\tilde T_{m-1}\mid \tilde T_m)}\\
        =&:\sum_{\tilde T_{1:M}\mid \tilde T_m=T_m}p(\tilde T_{1:M}\mid T_0)\tilde q(m\mid \tilde T_{0:M}).
    \end{aligned}
\end{equation*}
It turns out that $\tilde q(m\mid \tilde T_{0:M})$ is an unnormalized distribution  with a fixed normalizing constant as
$$\sum_{m=1}^M \log \frac{p(\tilde T_{m-1}\mid \tilde T_m, T_0)}{p_\mathrm{PDA}(\tilde T_{m-1}\mid \tilde T_m)}=\log \frac{p( \tilde T_{1:M}\mid T_0)}{p_\mathrm{PDA}(T_0, \tilde T_{1:M})}=-\log p_\mathrm{PDA}(T_0).$$
Therefore, calling $ q(m\mid \tilde T_{1:M})=\tilde q(m\mid \tilde T_{1:M})/\log(2N-3)!!$
we can simply sample from $\pi(T_m\mid T_0)$ by simple sampling a noising path $T_{1:M}\sim p(\tilde T_{1:M}\mid T_0)$
and selecting $m\sim q(m\mid T_{1:M})$.

\sloppy 
The wrinkle is that $q(m\mid T_{1:M})$ need not be positive!
The solution is to instead rejection sample from the proposal distribution
$$\hat \pi(T_m\mid T_0) \propto p(T_m\mid T_0)\mathbb E_{p(T_{m-1}\mid T_m, T_0)} [-\log p_\mathrm{PDA}(T_{m-1}\mid T_m)].$$
With a bit of algebra, one can show that $\tilde \pi(T_m\mid T_0)$ has normalizing constant $-\mathbb E_{p(T_{1:M}\mid T_0)}\log p_\mathrm{PDA}(T_{1:M}, T_0)=\log(2N-3)!!+\log|\mathrm{LE}(\texttt{int}(T_0))|$ and, as above,
\begin{equation*}
    \begin{aligned}
        \tilde \pi(T_m\mid T_0)
        =&\sum_{\tilde T_{1:M}\mid \tilde T_m=T_m}p(\tilde T_{1:M}\mid T_0)\hat q(m\mid \tilde T_{0:M}),
    \end{aligned}
\end{equation*}
where $\hat q(m\mid T_{0:M})=-\log p_\mathrm{PDA}(T_{m-1}\mid T_m)/(\log(2N-3)!!+\log|\mathrm{LE}(\texttt{int}(T_0))|).$
This gives Alg.~\ref{alg:bu-rejection}.

\begin{algorithm}[H]
\caption{Importance sampling the \bu{} noise schedule}
\label{alg:bu-rejection}
\begin{algorithmic}
\REQUIRE Tree $T_0$ with $N$ leaves
\REPEAT
    \STATE Draw a noising path $T_{1:M}\sim p(T_{1:M}\mid T_0)$
    \STATE Draw $m \in \{1, \dots, N-1\}$ with probability $q(m\mid T_{0:M}) \propto -\log p_\mathrm{PDA}(T_{m-1} \mid T_m)$
    \STATE Compute acceptance ratio $\rho \leftarrow \frac{D_\mathrm{KL}\!\left(p(T_{m-1} \mid T_m, T_0)\,\Vert\, q_\mathrm{PDA}(T_{m-1} \mid T_m)\right)}{\mathbb E_{p(T_{m-1} \mid T_m, T_0)}\left[ -\log q_\mathrm{PDA}(T_{m-1} \mid T_m)\right]}$
    \STATE Accept with probability $\rho$
\UNTIL{accepted}
\STATE \textbf{return} $T_m$
\end{algorithmic}
\end{algorithm}

Is this an efficient way to sample?
The average acceptance probability is $\mathbb E_{\hat \pi}\rho=\frac{\log(2N-3)!!}{\log(2N-3)!!+\log|\mathrm{LE}(\texttt{int}(T_0))|}\geq \frac{\log(2N-3)!!}{\log(2N-3)!!+\log(N-1)!}\geq 1/2.$
This lower bound shows we shouldn't expect to reject, on average, more than once per sample.

\subsection{\Mo{}}
\label{app:rb_mo}

\subsubsection{Distribution over paths}

The noising process removes one taxon at each step, selecting uniformly among the $N - m $ taxa present in $T_{m}$.
A path from $T_0$ to $T_{N}$ is a sequence of $N$ distinct taxa removed at each step.
The fully noised state $T_{N}$ is a tree containing only the pseudo-taxon.
Clearly we have

\begin{proposition}\label{prop:mo-uniform-paths}
Every removal sequence has the same probability:
$p(T_{1:N} \mid T_0) = 1/{N!}.$
\end{proposition}

\subsubsection{Calculating the target}

At state $T_m$, the tree contains $N - m$ taxa and the pseudo-taxon.
To reverse one noising step, we reattach one of the $m$ pruned taxa to the appropriate edge of $T_m$.
For each pruned taxon there is exactly one edge in $T_m$ such that grafting produces a tree consistent with $T_0$.

\begin{proposition}\label{prop:mo-reverse}
The reverse distribution is uniform over the $m$ valid re-attachments:
\[
p(T_{m-1} \mid T_m, T_0) = \frac{1}{m}.
\]
\end{proposition}

\begin{proof}
Follows from Prop.~\ref{prop:mo-uniform-paths}.
\end{proof}

\subsubsection{Noise schedule}
\label{app:mo_noise_schedule}

\paragraph{The PDA backwards steps}
By Prop.~\ref{prop:mo_expressivity}, the PDA Ford alpha model ($\alpha = 1/2$) produces uniform paths $T_{1:M}$ conditioned on $T_0$.
Therefore the reverse model $p_\mathrm{PDA}(T_{m-1} \mid T_m)$ picks a taxon to add and branch in the tree uniformly at random:
$$p_\mathrm{PDA}(T_{m-1} \mid T_m) = \frac{1}{m\left(2(N - m) - 1\right)}.$$

\paragraph{Sampling from $\pi$}

As for \bu{} we can write
$$\pi(T_m\mid T_0)
        =\sum_{\tilde T_{1:M}\mid \tilde T_m=T_m}p(\tilde T_{1:M}\mid T_0)q(m\mid T_0)$$
where $q(m\mid T_0)= \frac{\log \frac{p(\tilde T_{m-1}\mid \tilde T_m, T_0)}{p_\mathrm{PDA}(\tilde T_{m-1}\mid \tilde T_m)}}{\log(2N-3)!!}.$
In this case, for $m=1, \dots, N-1$, $q(m\mid T_0)=\frac{\log(2(N-m)-1)}{\log(2N-3)!!}$ is positive, so it's a probability distribution.
Therefore we can sample from $\pi$ by sampling a noising path and selecting a step according to $q$.
This gives Alg.~\ref{alg:mo-noise}.

\begin{algorithm}[H]
\caption{Importance sampling the \mo{} noise schedule}
\label{alg:mo-noise}
\begin{algorithmic}
\REQUIRE Tree $T_0$ with $N$ leaves
    \STATE Draw a noising path $T_{1:M}\sim p(T_{1:M}\mid T_0)$
    \STATE Draw $m \in \{1, \dots, N-1\}$ with probability $q(m\mid T_{0:M}) \propto \log(2(N-m)-1)$.
\STATE \textbf{return} $T_m$
\end{algorithmic}
\end{algorithm}

\subsection{\Td{}}

\subsubsection{Distribution over paths}

The noising process percolates taxa $v$ upward from depth $d_{T_0}(v)$ to the root, one level at a time.
Since only taxa at maximal depth are eligible, the process decomposes into \emph{rounds}: round~$d$ (for $d = d_{\max}, d_{\max}-1, \dots, 1$, where $d_{\max}$ is the maximum leaf depth in $T_0$) moves every taxon currently at depth~$d$ up to depth~$d-1$.

Let $V_d = \{u \in \texttt{int}(T_0) : d_{T_0}(u) \geq d\}$ be the taxa that are percolated upwards in round~$d$.
At the start of round~$d$, all of these taxa are at depth~$d$: those originally at depth exactly~$d$ have remained there, while those originally deeper have been brought to depth~$d$ by earlier rounds.
Within round~$d$, the taxa are moved from depth~$d$ to $d-1$ one at a time, each selected uniformly from the taxa still at depth~$d$.

\begin{proposition}\label{prop:td-uniform-paths}
Every sequence $T_{1:M}$ has the same probability:
$p(T_{1:M} \mid T_0) = 1/\!\prod_{d=1}^{d_{\max}} |V_d|!.$
\end{proposition}

\begin{proof}
Within round~$d$, each of the orderings of the taxa in $V_d$ is equally likely.
\end{proof}
Note $|V_d|$ is the same as $M_{d-1}$ from Prop.~\ref{prop:td_expressivity}.

\subsubsection{Calculating the target}

At state $T_m$, some rounds have been completed and one is partially processed.
Let $d$ be the depth of the current round (the maximal depth of any taxon in~$T_m$).
Let $ M_m = |\{v : d_{T_m}(v) = d\!-\!1 \text{ and } d_{T_0}(v) \geq d\}|$
be the number of taxa that have percolated up in $T_m$ in round $d$.

\begin{proposition}\label{prop:td-reverse}
The reverse distribution is uniform over $M_m$ valid downward moves:
\[
p(T_{m-1} \mid T_m, T_0) = \frac{1}{M_m}.
\]
\end{proposition}

\begin{proof}
Follows from Prop.~\ref{prop:td-uniform-paths}.
\end{proof}

\subsubsection{Noise schedule}
\label{app:td_noise_schedule}

\paragraph{The PDA backwards steps}

By Prop.~\ref{prop:td_expressivity}, the PDA single-taxa splitting model ($\beta = -3/2$) also produces uniform paths $T_{1:M}$ conditioned on $T_0$.
Therefore the reverse model $p_\mathrm{PDA}(T_{m-1} \mid T_m)$ picks a taxon at the current depth at random and percolates down according to $f$, following the procedure described in Alg.~\ref{alg:td-generate}.

\paragraph{Sampling from $\pi$}
Say $T_m$ is on round $d$ with $a_u, b_u$ nodes percolated for every node $u\in U_d$, and $\bar a_u$, $\bar b_u$ nodes not yet percolated which belong to the group with $a_u$ and $b_u$ taxa in $T_0$.
Call $M_u=\bar a_u+\bar b_u$ the number of leaves at $u$ in $T_m$, $|\texttt{leaves}(u)|$, $N_u=\bar a_u+\bar b_u+a_u+b_u$ the number of leaves in the subtree rooted at $u$ in $T_0$, $n_\mathrm{leaves}(u)$, and $A_u=a_u+\bar a_u$, $B_u=b_u+\bar b_u$.

At every $u\in\texttt{int}(T_0)$, $p_\mathrm{PDA}(A_u, B_u):=\frac{B(A_u+\beta+1, B_u+\beta+1)}{B(\beta+1, \beta+1)-2B(N_u+\beta+1, \beta+1)}$ is, intuitively, the amount of information created at node $u$.
In particular note that $\prod_{u\in\texttt{int}(T_0)}p_\mathrm{PDA}(A_u, B_u)=p_\mathrm{PDA}(T_0)=1/(2N-3)!!.$
First we sample a $u$ according to how much information is created at that node.
Define, for $u\in\texttt{int}(T_0)$, 
$$\pi_u(T_m\mid T_0)= \frac{1}{-\log p_\mathrm{PDA}(A_u, B_u)}\sum_{T_{m-1}\text{ decreases }M_u}p(T_m, T_{m-1}\mid T_0) \log\frac{p(T_m\mid T_{m-1}, T_0)}{p_\mathrm{PDA}(T_m\mid T_{m-1})}.$$
Then clearly $\pi(T_m\mid T_0)=\sum_{u\in\texttt{int}(T_0)}\frac{-\log p_\mathrm{PDA}(A_u, B_u)}{\log (2N-3)!!}\pi_u(T_m\mid T_0)$.
Now say $u\in U_d$, and $m_1, \dots, m_2$ are all in stage $d$, so $T_{m_1}$ is the start of stage $d$ and $T_{m_2}$ is the end, both deterministic.
Then $\log (2N-3)!!\sum_{T_m\text{ depth }d} \pi_u(T_m\mid T_0)$
\begin{equation*}
    \begin{aligned}
        =&\sum_{T_m\text{ depth }d}\sum_{\ \ T_{m-1}}p(T_m, T_{m-1}\mid T_0)\mathbbm 1(T_{m-1}\text{ decreases }M_u)\log\frac{p(T_m\mid T_{m-1}, T_0)}{p_\mathrm{PDA}(T_m\mid T_{m-1})}\\
        =&\mathbb E_{T_{m_1+1:m_2}\sim p(T_{m_1+1:m_2}\mid T_0)}\sum_{m=m_1+1}^{m_2}\mathbbm 1(T_{m-1}\text{ decreases }M_u)\log\frac{p(T_m\mid T_{m-1}, T_{m_2})}{p_\mathrm{PDA}(T_m\mid T_{m-1})}\\
        =&-\mathbb E_{T_{m_1+1:m_2}\sim p(T_{m_1+1:m_2}\mid T_0)}\log p_\mathrm{PDA}(A_u, B_u)\\
        =& -\log p_\mathrm{PDA}(A_u, B_u).\\
    \end{aligned}
\end{equation*}
The third equality follows from summing the log probabilities of revealing the splitting of $u$ one node at a time.
Therefore, $\pi_u(T_m\mid T_0)$ is a probability distribution!
To sample from $\pi$ therefore we can sample $u$ in proportion to $-\log p_\mathrm{PDA}(A_u, B_u)$ then sample from $\pi_u$.

Now how do we sample from $\pi_u$?
Calling $I_u(a_u, b_u)=-\frac{\bar a_u}{M_u}\log f(a_u, b_u\mid M_u)-\frac{\bar b_u}{M_u}\log f(b_u, a_u\mid M_u)$ (and $I_u(0,0)=0$) we have, with a bit of algebra,
$$\pi_u(T_m\mid T_0)\propto p(T_m\mid T_0)\frac{M_u}{\sum_{u'}M_{u'}}I_u(a_u, b_u).$$
If $u\in U_d$, then $p(T_m\mid T_0)$ is a hypergeometric over $(M_u, a_u)_{u\in U_d}$:
\begin{equation*}
    \begin{aligned}
        \binom{\sum_u N_u}{\sum_u M_u}^{-1}\prod_{u\in U_d}\binom{A_u}{a_u}\binom{B_u}{b_u}=&\left(\binom{\sum_u N_u}{\sum_u M_u}^{-1}\prod_{u\in U_d}\binom{N_u}{M_u}\right)\\
        &\prod_{u\in U_d}\times\mathrm{HyperGeom}(N_u, A_u, N_u-M_u)(a_u)\\
        =&\left(\binom{\sum_u N_u}{\sum_u M_u}^{-1}\binom{N_u}{M_u}\binom{\sum_{u'\neq u}N_{u'}}{\sum_{u'\neq u}M_{u'}}\right)\\
        &\times\mathrm{HyperGeom}(\sum_{u'\neq u}N_{u'}, (A_{u'})_{u'\neq u}, \sum_{u'\neq u} M_{u'})(M_{u'})_{u'\neq u}\\
        &\times\prod_{u\in U_d}\mathrm{HyperGeom}(N_u, A_u, N_u-M_u)(a_u).
    \end{aligned}
\end{equation*}
This also shows that once we've sampled the variables affected by the weight in front of $p(T_m\mid T_0)$ --- $M_u, \sum_u M_u$ and $a_u$ --- then we can sample $(M_{u'})_{u'\neq u}$ and $(a_{u'})_{u'\neq u}$ from the hyper geometrics above.
Now we only need to sample from a distribution proportional to
$$\frac{\binom{\sum_{u'\neq u} N_{u'}}{\sum_{u'\neq u} M_{u'}}\binom{N_u}{M_u}}{\binom{\sum_u N_u}{\sum_u M_u}}\frac{M_u}{\sum_{u'}M_{u'}}\hat\pi_u(a_u, M_u)$$
where $\hat\pi_u(a_u, M_u)=I_u(a_u, b_u)\mathrm{HyperGeom}(N_u, A_u, N_u-M_u)(a_u).$
But that pre-factor is
$$\frac{\binom{\sum_{u'\neq u} N_{u'}}{\sum_{u'\neq u} M_{u'}}\binom{N_u}{M_u}}{\binom{\sum_u N_u}{\sum_u M_u}}\frac{M_u}{\sum_{u'}M_{u'}}=\binom{\sum_{u'\neq u} N_{u'}}{\sum_{u'\neq u} M_{u'}}\frac{B(\sum_u M_u, \sum_u (N_u-M_u)+1)}{B(M_u, N_u-M_u+1)}$$
which is a beta binomial $$\mathrm{BetaBin}\left(\sum_{u'\neq u} N_{u'}, M_u, N_u-M_u+1\right)\left(\sum_{u'\neq u} M_{u'}\right)$$
To sample from $\hat\pi_u(a_u, M_u)$ we build the $A_u\times B_u$ table explicitly.
This gives Alg.~\ref{alg:td-sample}.

\begin{algorithm}[H]
\caption{Importance sampling the \td{} noise schedule}
\label{alg:td-sample}
\begin{algorithmic}
\REQUIRE Tree $T_0$ with $N$ leaves
    \STATE Draw $u\sim\mathrm{Categorical}(-\log p_\mathrm{PDA}(A_u, B_u)/\log(2N-3)!!)$. Call $d=d_{T_0}(u)$.
    \STATE Draw $a_u, M_u\sim \mathrm{Categorical}\left(\hat\pi_u(a_u, M_u)/\sum_{a, M}\hat\pi_u(a, M)\right)$.
    \STATE Draw $\sum_{u'\in U_d\setminus\{u'\}} M_{u'}\sim\mathrm{BetaBin}\left(\sum_{u'\in U_d\setminus\{u'\}} N_{u'}, M_u, N_u-M_u+1\right)$.
    \STATE Draw $(M_{u'})_{u'\in U_d\setminus\{u'\}}\sim\mathrm{HyperGeom}\left(\sum_{u'\in U_d\setminus\{u'\}}N_{u'}, (A_{u'})_{u'\in U_d\setminus\{u'\}}, \sum_{u'\in U_d\setminus\{u'\}} M_{u'}\right)$.
    \STATE Draw $(a_{u'})_{u'\in U_d\setminus\{u'\}}\sim \prod_{u'\in U_d\setminus\{u'\}}\mathrm{HyperGeom}(N_{u'}, A_{u'}, N_{u'}-M_{u'})$.
    \STATE For each $u\in U_d$ assign $a_u$ taxa uniformly at random to one child and $b_u$ to the other.
\STATE \textbf{return} $T_m$
\end{algorithmic}
\end{algorithm}

\section{Architecture details}
\label{app:architecture}

Our model $q_\theta$ must condition on the sequences $X_{1:N}$ and current state of the tree $T_m$. To condition on the sequences, we adapt the EvoPF architecture from Phyloformer 2 \citep{Blassel2025-vq}.
EvoPF embeds $N$ sequences of length $L$, $X_{1:N}$, into a tensor of size $[N, L, 128]$ and performs axial attention over the dimensions $N$ and $L$ separately.
After a trunk made of 12 layers, we turn embeddings from each pair of node into a set of logits.

\subsection{An Embedding trunk that conditions on the tree $T_m$}
We pass sequence embeddings through twelve EvoPF layers modified to condition on the tree structure $T_m$.
We make two modifications that allow us to condition on $T_m$:

\paragraph{Internal nodes}

\sloppy We include internal nodes of $T_m$ in our calculation, so our tensor is of size $[N+\#\texttt{internal\_nodes}(T_m), L, 128]$.
Say $h_v^{(k)}\in\mathbb R^{(L+1)\times 128}$ is the embedding of node or leaf $v$ at layer k.
The $+1$ is from an extra position we also add to the beginning of each sequence which will allow us to read out a single vector.

For an internal node $v$ we initialize $h_v^{(0)}$ using topological information and information about nearby leaves.
\begin{itemize}
    \item For topological information on $v$ we use the depth and the number of leaves in the subtree $\texttt{depth}(v), |\mathrm{leaves}(v)|$.
    We call these features $f_{\mathrm{topo}, v}$.
    \item For every position $\ell$, we average the embeddings of all leaves in the subtree of the node
    $\frac{1}{|\mathrm{leaves}(v)|}\sum_{w \in \mathrm{leaves}(v)} h_{w,\ell}^{(k)}$.
    We call these features $f_{\mathrm{subtree}, v, \ell}^{(0)}$.
    \item For every position $\ell$, we average the embeddings of all leaves according to the Dirichlet energy features from \citet{Xie2023-ed}.
    This can be summarized as multiplying by a $|T_m|\times |T_m|$ matrix $A$: $\sum_{w \in T_m} A_{v, x}h_{w,\ell}^{(k)})$.
    We call these features $f_{\mathrm{Dirichlet}, v, \ell}^{(0)}$.
\end{itemize}
We feed all of these features through an MLP at every position
$$h_{v,\ell}^{(0)} \leftarrow \mathrm{MLP}_{0, \theta}\!\left(f_{\mathrm{topo}, v}, 
f_{\mathrm{subtree}, v, \ell}^{(0)}, f_{\mathrm{Dirichlet}, v, \ell}^{(0)}
\right).$$

After each layer, each internal node $v$ has graph information injected into it as
$$h_{v, \ell}^{(k)} \leftarrow h_{v, \ell}^{(k)} + \mathrm{MLP}_{k, \theta}\!\left(f_{\mathrm{topo}, v}, 
f_{\mathrm{subtree}, v, \ell}^{(k)}, f_{\mathrm{Dirichlet}, v, \ell}^{(k)}
\right)$$
where $\mathrm{leaves}(v)$ denotes the leaf descendants of $v$ in $T_m$.

\paragraph{Attention bias}
During attention over $N$ we bias the attention using the graph distance in $T_m$ (which can be $\infty$):
    for nodes $v, u$, position $\ell$, and head $h$,
    $$\mathrm{Attention}_{v, u, \ell, h}\gets \mathrm{Attention}_{v, u, \ell, h}+\mathrm{embed}_{\theta, h}(\texttt{dist}(u, v\mid T_m)).$$
    If $u$ and $v$ are leaves then we also bias the attention using their Hamming distance:
    $$\mathrm{Attention}_{v, u, \ell, h}\gets \mathrm{Attention}_{v, u, \ell, h}+\mathrm{MLP}_{\theta, h}(\texttt{sinusoidal}(\texttt{hamming}(X_u, X_v))).$$

\subsection{Readout}
After the trunk we arrive at a vector $\texttt{embed}_\theta(v)=h^{(12)}_{v, \ell}\in\mathbb R^{128}$ for every vertex or leaf $v$ in $T_m$ and position $\ell$ which depends on both $T_m$ and $X_{1:N}$.
Now we need a strategy to first pool over the length dimension and then a strategy for each paradigm to get a set of scores for each possible $T_{m-1}$ which we then normalize.

\paragraph{Covariance pooling along the length dimension}

EvoPF~\citep{Blassel2025-vq} pools over the length dimension by taking a mean.
We found this resulted in embeddings that were nearly identical for different nodes, especially when sequences were long.
This made training extremely unstable.

We suggest a fix inspired by classical phylogenetics methods.
To motivate the method, first note at initialization, for a leaf $v$, $h^{(12)}_{v,\ell}$ is roughly the embedding of the letter at position $\ell$ in $v$.
When we take a mean, $\frac 1 {L+1}\sum_{\ell}h^{(12)}_{v,\ell}$ then our embedding is simply the average composition of $v$.
For randomly generated trees, each nucleotide or amino acid appears in roughly the same frequency across sequences, so this embedding is nearly identical across all $v$.
To build useful embeddings, phylogeneticists instead build a matrix of average similarity across sequences, which we emulate by taking $$\texttt{embed}(v, w):=\frac 1 {L+1}\sum_{\ell}h^{(12)}_{v, \ell}\circ h^{(12)}_{w, \ell}-\left(\frac 1 {L+1}\sum_{\ell}h^{(12)}_{v, \ell}\right)\circ\left(\frac 1 {L+1}\sum_{\ell}h^{(12)}_{w, \ell}\right)$$ for all pairs $v, w\in T_m$.
At initialization, entries of $h_{u, w}$ are large for leaves $v, w$ which have similar sequences, and small for those with dissimilar sequences.
Note this pooling step can not straightforwardly be dropped into EvoPF -- it is specific to our stretegy of iteratively building up the tree.

After pooling we apply a self-attention layer to $\{h_{v, w}\}_{v, w\in T_m}.$
Note the number of pairs of sequences is typically small compared to the length $L$, so this computation has negligible cost in our overall architecture.

\paragraph{Scoring $T_{m-1}$}

For all pairs of nodes $v, w\in T_m$ we now have a vector of similarities $\texttt{embed}(v, w)$.
We simply choose the right pair of nodes to compare to compute un-normalized logits.

\begin{itemize}
    \item \bu{}: for each pair of subtree roots $v, w$ in the current forest,
    $$\mathrm{logit}(v, w)=\texttt{MLP}\left(\texttt{embed}(v, w)\right).$$

    \item \mo{}: here we need to score each unattached taxon $t$ against each parent-to-child edge $(v, w)$ in $T_m$.
    To do so, before pooling the length dimension we define edge embeddings
    $$h^{(12)}_{(v, w), \ell}=W_\theta[h^{(12)}_{v, \ell}, h^{(12)}_{w, \ell}]$$
    for a matrix $W_\theta\in\mathbb R^{128\times 256}$ so that our pooling step gives us $\texttt{embed}(t, (v, w))$.
    Finally we take
    $$\mathrm{logit}(t, (v, w))=\texttt{MLP}\left(\texttt{embed}(t, (v, w))\right).$$

    \item \td{}: each taxon $t$ at the deepest layer of $T_m$ sits at a host
    internal node with left and right children $\texttt{left}(t), \texttt{right}(t)$;
    the percolation logits are
    \begin{align*}
    \mathrm{logit}_{\text{left}}(t)
    &= \texttt{MLP}\left(\texttt{embed}(t, \texttt{left}(t))\right),\\
    \mathrm{logit}_{\text{right}}(t)
    &= \texttt{MLP}\left(\texttt{embed}(t, \texttt{right}(t))\right).
    \end{align*}
    if $t$ has only one child then the embedding of the missing child is set to be a fixed vector plus $-\texttt{embed}(t, \texttt{child}(t))$.
    Note for top-down there can be multiple taxa at a node;
    we therefore represent nodes and taxa that may occupy a node separately in our trunk.
\end{itemize}

\section{Experiment details}

\subsection{Additional information on training data}
\label{app:evolutionary_models}

We simulate data from three priors by first simulating a tree, and then simulating sequences conditional on the tree.

\subsubsection{Dataset size}

For each set of data point, we sample
$$N\sim\mathrm{Uniform}(25, 65), L\sim\mathrm{Uniform}(300, 2600).$$
Then we generate a tree with $N$ species and sequences with length $L$ before any insertions or deletions.

\subsubsection{Trees}

MrBayes~\citep{Huelsenbeck2001-mo} and CAT~\citep{Lartillot2004-ri} both have uniform (proportional to distinguishable arrangements (PDA)) priors on tree topologies.
We simulate PDA trees using the Ford-$\alpha$ process with $\alpha=1/2$~\citep{Ford2005-gf}.

\subsubsection{Sequences and branch lengths}\label{app: ali priors}
All simulations were done using AliSim~\citep{Ly-Trong2022-nt}, part of the IQ-TREE software used under the GPL-2.0 license.
We use $L$ to denote the length of the sequences.

\paragraph{Simple model}
To simulate from the Jukes-Cantor model, we passed \texttt{JC} to the model flag in AliSim.
We simulate branch lengths independently from $ \mathrm{Exp}(1)/10$ to match the default MrBayes prior.

\paragraph{Complex model}
We simulate from a CAT model~\citep{Lartillot2004-ri} which models site-wise biases in mutation rates in amino acid sequences.
This model has branch lengths $(r_{u\to v})_{u, v}$, class diversity parameter $\alpha\in(0, \infty)$, and a symmetric generator matrix $Q$ representing relative mutation rates as global parameters.
Then each site $\ell$ mutates according to its rate $g_\ell$ and its stationary distribution $\vec\pi_\ell$.
In particular, we (1) modify $Q$ to get a stationary rate $\vec\pi_\ell$, $\tilde R_\ell = Q\mathrm{diag}(\vec \pi_\ell)-\mathrm{diag}(Q\vec \pi_\ell)$; (2) normalize the rates $R_\ell=\tilde R_\ell/(-\vec\pi_\ell^T\mathrm{diag}(\tilde R_\ell))$; and finally (3) multiply by the rate $Q_\ell=g_\ell\times R_\ell$.

We simulate parameters from a CAT prior as used in PhyloBayes MPI~\citep{Lartillot2013-fj}.
See the manual at \url{https://github.com/bayesiancook/phylobayes/tree/master}.
\begin{itemize}
    \item $Q_{ij}\sim\mathrm{Exp}(1)$ for $i<j$. This matches the \texttt{-gtr} flags.
    \item $\alpha\sim\mathrm{Exp}(1)\times 10$.
    \item $r_{u\to v}\sim\mathrm{Exp}(1)/\mu$ independently where $\mu\sim\mathrm{Exp}(1)\times 10$.
    \item $(g_\ell)_{\ell=1}^L\sim \mathrm{Gamma}(\theta, 1)/\theta$, discretized to 4 components, which has mean $1$ and variance $1/\theta$, where $\theta\sim\mathrm{Exp(1)}$.
    \item $(\vec\pi_{\ell})_{\ell}^L\sim\mathrm{DirichletProcess}(\alpha, \mathrm{Dirichlet(\eta_1, \dots, \eta_{20})})$ where each $\eta_b\sim\mathrm{Exp}(1).$
    This matches the \texttt{-cat} flag.
    We simulate class probabilities using the standard stick-breaking procedure up to 100 classes. 
\end{itemize}
Finally we can simulate this prior using AliSim by passing \texttt{GTR20\{...\}+FMIX\{...\}+G4\{...\}} filling in the sampled parameters for $Q$, the probabilities and classes from the sampled stick breaking process, and $\theta$.

\paragraph{Intractable model}
To simulate from a realistic prior including indels, we first simulate from a standard model of DNA mutation, \texttt{GTR\{q\_matrix\}+GC\{g\_std\}+I\{invar\_rate\}}, 
where \texttt{q\_matrix} is sampled from a Dirichlet as in the CAT model, \texttt{g\_std}$\ \sim\mathrm{Exp}(1)$, and \texttt{invar\_rate}$\ \sim\mathrm{Uniform}(0, 0.4)$.
AliSim will simulate the stationary DNA distribution, what we called $\vec\pi$ in the CAT model, from an empirical distribution.
We also simulate branch lengths from $\mathrm{Exp}(1)/10.$

To add insertions and deletions, we pass \texttt{--indel ins\_rate,del\_rate} where we simulate \texttt{ins\_rate}$\ \sim\mathrm{Uniform}(0, 0.01)$ and \texttt{del\_rate}$\ \sim\mathrm{Uniform}(0, 0.05)$.
This simulates insertions up to 1\% and deletions up to 5\% as frequently as substitutions; AliSim then samples insertion and deletion sizes from an empirically fit Zipfian distribution~\citep{Fletcher2009-yg}.

After simulating, we align sequences using \texttt{kalign3}~\citep{Lassmann2019-on} with default parameters before performing inference.
Note models do not see the ground truth insertion and deletion labels generated during simulation, only those inferred by \texttt{kalign3}.

\paragraph{Missing-ness}
To simulate missing-ness, we draw a $p\sim\mathrm{Beta}(1.25, 8.75)$ and give each non-gap position an independent $\mathrm{Bernoulli}(p)$ chance of being masked.

\subsection{Real sequence data}
\label{app:real_data}

We downloaded sequence datasets from~\citet{Lakner2008-dd}.
We use DS1 originally taken from~\citet{Hedges1990-tm}, DS6 original taken from~\citet{Brower2000-um}, and DS9 originally taken from~\citet{Rossman2001-tz}.

\subsection{Inference with likelihood-based methods}

We compared our models to MCMC sampling with MrBayes and a maximum likelihood based method, IQ-Tree.

\subsubsection{MrBayes}
We used code from \url{https://nbisweden.github.io/MrBayes/download.html} under a GPL-3.0 license.
We ran MrBayes with a uniform stationary distribution \texttt{prset statefreqpr=fixed(equal)} and exponential branch lengths \texttt{prset brlenspr=unconstrained:exponential(10.0)}.
Our gold standard runs were run until the convergence criterion, the standard deviation in subsplit probabilities across two runs, was less than $0.01$, sampling every 1000 steps.

\subsubsection{IQ-Tree}
We use IQ-Tree3~\citep{Wong2026-rc} code from \url{https://github.com/Cibiv/IQ-TREE} under a GPL-2.0 license.
We ran IQ-tree with 16 CPUs with the flag \texttt{-nt AUTO} and otherwise used default settings.
We ran IQ-Tree with \texttt{-m JC} for both DNA datasets and \texttt{-m Poisson} for the protein dataset.
Our ``flexible'' IQ-Tree model was \texttt{-m GTR+G+I} for the DNA dataset and \texttt{-m LG+F+G+I} for the protein dataset.
We attempted to use the more flexible \texttt{-m LG+C60+F+G+I} for the protein dataset but this took over an hour to generate a single tree for a dataset with $N=45$ and $L=1000$.
To generate $8$ trees to build a consensus, we ran IQ-Tree with $8$ different seeds.

\subsection{Evaluation metrics}
\label{app: metrics}

We describe how we compare trees and their distributions using metrics derived from the Robinson-Foulds distance.

\subsubsection{Robinson-Foulds distance}

The Robinson-Foulds (RF) metric is often used for comparing trees due to its ease of computation~\citep[Ch. 30]{Felsenstein2004-xi}.
The metric can be used to compare rooted and unrooted trees.
Since JC and CAT are reversible, the position of the root is not identifiable, so we use the unrooted RF distance to compare all trees.

To compute the RF distance, we compute the partition features of each tree:
for a binary partition $P_1, P_2\subset \{1, \dots, N\}$ a tree has that feature if there is an edge in the unrooted tree separating all nodes $P_1$ from all nodes $P_2$.
Then the distance between two trees is simply the symmetric difference of their features.

\subsubsection{Comparing tree distributions using RF MMD}\label{app: mmd}

To compare two distributions $p$, $q$, using a kernel $k$, we can calculate
$$\mathrm{MMD}_k^2(p, q)=\mathbb E_{T_1\sim p, T_2\sim p}k(T_1, T_2) + \mathbb E_{T_1\sim q, T_2\sim q}k(T_1, T_2) - 2\mathbb E_{T_1\sim p, T_2\sim q}k(T_1, T_2).$$

We can build kernels using the RF distance:
we can think of RF as embedding a tree $T$ into $v(T)\in\mathbb R^{2^{N-1}}$ with dimensions indexed by binary partitions of $\{1, \dots, N\}$, where, for a partition $P_1, P_2\subset \{1, \dots, N\}$, $v(T)_{\{P_1, P_2\}}=1$ if there is an edge separating all nodes $P_1$ from all nodes $P_2$.
Then the RF distance between trees $T, T'$, is simply $\Vert v(T)-v(T')\Vert$.

We could compare distributions with the linear kernel $k(T, T')=v(T)^Tv(T')$;
the MMD in this case is simply the distance between the vectors of subsplit probabilities $\Vert \mathbb E_{T\sim p}v(T) - \mathbb E_{T\sim q}v(T)\Vert$.
However, two distinct distributions can have zero MMD simply by matching the average subsplit probabilities.
To more flexibly detect distributional differences we look at a more flexible kernel $k(T, T')=\exp(-\Vert v(T)-v(T')\Vert^2/2\sigma^2)$ where we chose $\sigma^2=(N-3)/2$ so $\Vert v(T)-v(T')\Vert^2/2\sigma^2$ is twice the ``relative RF'' distance between $T$ and $T'$.

\subsubsection{JC marginal likelihood estimate}\label{app: JC ML}

We wanted to estimate the average entropy $-\mathbb E_{T_0,X_{1:N}}\log p(T_0\mid X_{1:N})$ on the simulated JC data where $p$ is the true prior model.
Then we could bound the absolute error of our models as
$$\mathbb E_{X_{1:N}}D_\mathrm{KL}(p(T_0\mid X_{1:N})\Vert q_\theta(T_0\mid X_{1:N}))\leq \mathbb E_{T_0,X_{1:N}}\log p(T_0\mid X_{1:N})-\mathrm{ELBO}(q_\theta).$$
The entropy isn't tractable so we estimated it two ways.

First we ran MrBayes with sequences $X_{1:N}$ to get a large number of posterior samples $T^1, \dots T^M$ and simply estimated $p(T_0\mid X_{1:N})\approx \hat p=\frac {\#\{m\mid T^m=T_0\}} M.$
This is an under-estimate as $\mathbb E\log \hat p\approx \log p(T_0\mid X_{1:N}) -\frac 1 2\left(\frac{\mathrm{std}\hat p}{\hat p}\right)^2$.
After a simulation, if $\frac 1 2\left(\frac{\mathrm{std}\hat p}{\hat p}\right)^2$ was greater than $0.1$ then we re-ran the simulation for twice as long, repeating until the estimated error was less than $0.1$.
We ran MCMC with \texttt{mcmc nruns=4 nchains=4 ngen=2000000 samplefreq=500} and increased \texttt{ngen} as necessary.

Second we estimated the likelihood by first decomposing
$$\log p(T_0\mid X_{1:N}) = \log p(T_0) + \log p(X_{1:N}\mid T_0) - \log p(X_{1:N}).$$
Under the PDA prior, $\log p(T_0)=-\log\#\{\text{trees of size }N\}$.
The next two terms require marginalizing over branch lengths, and $\log p(X_{1:N})$ also involves marginalizing over trees. 
These were estimated using the stepping-stone algorithm in MrBayes.
This algorithm again gives an under-estimate for each of these quantities and reports a standard deviation;
    again, if the standard error was larger than $0.1$ then we re-ran the simulation for twice as long.
We ran the stepping stone algorithm, with or without a fixed tree, with \texttt{ss ngen=1000000 nruns=4 nchains=4 samplefreq=500 nsteps=50 alpha=0.4} and increased \texttt{ngen} as necessary.

For 100 sampled tree-sequence pairs, our first method failed to sample the true tree after $10^9$ generations for 15/100 pairs.
On the remaining pairs, which typically exclude those high-entropy trees, our first method gave an average of $9.4$ and the second method gave an average of $8.9$; these were further apart than our $0.1$ threshold likely because the simple approximations above were not accurate.
On all the data, our second method gave an estimate of $10.8$ nats with a standard error of $0.5$.
We therefore report a value of $>10$ nats in the main text.
In principle we could perform the same estimation with CAT, but the computational cost is prohibitive: even this simulation took roughly 2 hours on 8 CPUs to evaluate a single tree-sequence pair.

\subsection{Phyloformer 2}

We trained Phyloformer 2~\citep{Blassel2025-vq} for 48 hours on two A100s --- the same as our models.
We used the training architecture and parameters from the supplementary information of \url{https://openreview.net/forum?id=18pFYYqFiG}, including the branch-length canonicalization.
However we removed the branch length prediction head.
In particular the model was trained with a learning rate of $0.0001$, with a batch size of $64$ and precision of \texttt{fp32}.
The settings are similar to our model except we use lower precision and half the batch size.
In the original code the trained on data with a fixed $N$ and $L$;
    we modified their code to train on varied $N$ and $L$ every batch just like our model.

\subsection{Quartformer}
\label{app:quartformer}

We trained Quartformer \citep{Su2025-xt} models for 48 hours on two A100 GPUs, consistent with our models. The architecture details follow from their paper, including sampling blocks of 16 taxa and taking a 3-way cross-entropy loss for all 1820 quartets. Our only changes are:
\begin{itemize}
    \item We use a a learned column MLP mean-pooled over sites rather than their histogram approach, which would be intractable for the larger vocabulary of amino acids
    \item Replacing the post-normalization with pre-normalization blocks and final LayerNorm, to improve training stability.
\end{itemize}

To sample a tree, we sample 16-taxa blocks until each taxon appears at least 40 times and store the predicted outputs. Following \citet{Su2025-xt}, the predictions are then assembled using wQFM.

\section{Supplementary results}

\subsection{Training strategy ablations}\label{app: ablate math}

Here we ablate choices in our training strategy.

\paragraph{Noise schedule}
Above we derived noise schedules so that de-noising steps were generated in proportion to where information is generated when $p(T_0)$ is the PDA distribution; this makes the variance in our loss estimate $0$ at the beginning of training.
Compared to sampling $m$ uniformly however, Fig.~\ref{fig:IS ablation} there is little, if any, benefit.

\begin{figure}[H]
    \centering
    \includegraphics[width=0.95\linewidth]{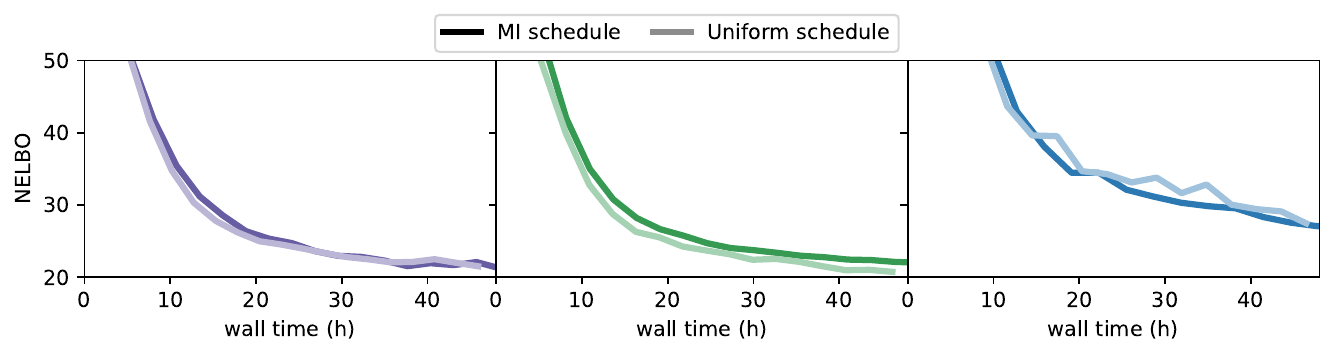}
    \caption{\textbf{The bespoke mutual information schedules do not make much of a difference when training.}
    On the ``complex'' data, with the same architecture and training parameters we train models for 24 hours on 2 A100s.
    We train models either with our derived mutual information schedule or by sampling the sampling step $m$ uniformly between $1$ and $M$.
    }
    \label{fig:IS ablation}
\end{figure}

We hypothesize this is because during training the model diverges very sharply away from PDA which was the basis of the noising process, starting from a trivial average likelihood of near 150 nats (the entropy of PDA on our data) to around 20 nats.
Since PDA is the best approximation to $p(T_0\mid X_{1:N})$ that does not use any information about $X_{1:N}$, a better noising distribution must include information about the sequences themselves.
On one hand, this suggests that a better schedule would require even more complicated analysis, but on the other, it suggests we can be satisfied with the simple uniform schedule.

\paragraph{Prior corrected initialization}

In App.~\ref{app:expressivity} we showed that all three paradigms can express distributions between PDA and Yule.
We consider PDA data and note that, while our \mo{} paradigms is initialized at the PDA distribution, \td{} and \bu{} are not.
We hypothesized that this may harm training, so we multiplied the outputs of our neural networks by our derived PDA posterior $\log p_{\mathrm{PDA}}(T_{m-1}\mid T_m)$ for \bu{} and \td{}.
However, Fig.~\ref{fig:prior correctes ablation} we see that this improves \td{}, whose PDA posterior is somewhat complicated, early in training.
However, after one day of training there is no detectable benefit to doing so, suggesting these paradigms can learn the PDA correction in practice.

\begin{figure}[H]
    \centering
    \includegraphics[width=0.7\linewidth]{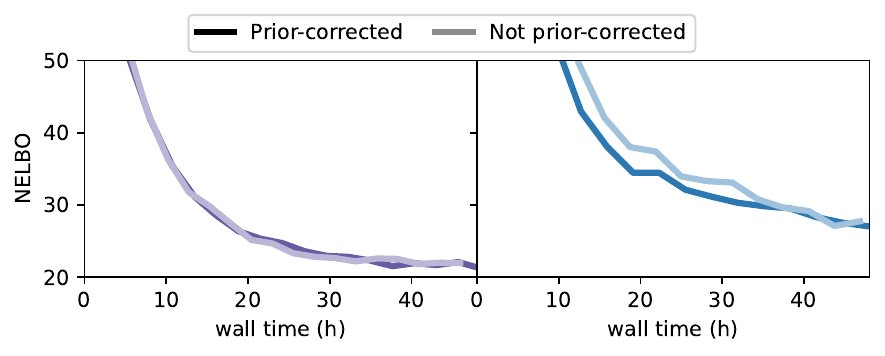}
    \caption{\textbf{The prior corrected outputs do not make much of a difference when training.}
    On the ``complex'' data, with the same architecture and training parameters we train models for 24 hours on 2 A100s.
    We train \td{} or \bu{} models either by multiplying outputs by $\log p_{\mathrm{PDA}}(T_{m-1}\mid T_m)$ or leaving them biased at initialization..
    }
    \label{fig:prior correctes ablation}
\end{figure}

\subsection{Architecture ablations}\label{app: arch ablations}

\paragraph{Graph conditioning}
We ran ablations demonstrating the importance of our two strategies for injecting graph information into the model.
We removed $f_{\mathrm{subtree}, v, \ell}^{(0)}$ and $f_{\mathrm{Dirichlet}, v, \ell}^{(0)}$ from the initialization of internal node initializations, and we removed $\mathrm{embed}_{\theta, h}(\texttt{dist}(u, v\mid T_m))$ from each attention layer.
In Fig.~\ref{fig:graph ablation}, on the ``complex'' data, with the same architecture and training parameters we trained these models for 48 hours.

\paragraph{Pooling strategy ablations}
We show that an architecture with mean pooling performs substantially worse than one with covariance pooling.
For a bottom-up model, we replace our readout with a mean $\texttt{embed}(v)=\frac{1}{L+1}\sum_\ell h^{(12)}_{v, \ell}$.
We apply two dense layers and then read out, for each pair of subtree roots $v, w$ in the current forest,
    $$\mathrm{logit}_\theta(u, v)=\texttt{MLP}\left([\texttt{embed}(v)+\texttt{embed}(w), \texttt{embed}(v)\circ \texttt{embed}(w)]\right).$$
We trained these models in Fig~\ref{fig:graph ablation}, on the ``complex'' data, with the same architecture and training parameters  for 48 hours.

\subsection{Phyloformer 2 ablations}\label{app:pf2 ablations}

In the main text we saw that our models significantly outperform Phyloformer 2.
We proposed that this was due to the increased flexibility of our models and our improved objectives:
    our models take in a partial tree $T_m$ and do a full forward pass to predict a posterior over the next move $T_{m-1}$, meanwhile Phyloformer 2 constructs fixed embeddings for every sequence $X_{1:n}$ in one forward pass and these embedding vectors $\mathbb R^{N\times128}$ must parameterize a distribution over trees that must predict a single ``correct'' path.
However, there are two other plausible hypotheses for our improvements we ablate here.

\paragraph{Hypothesis 1: Our Phyloformer 2 ELBO is pessimistic}
Phyloformer 2 originally also predicted branch lengths.
Without this prediction, the branch-length based canonicalization introduces an irreducible error in the ELBO.
Note this can only explain the ELBO result and not the MMD or RF distance results where we compare actual samples from Phyloformer 2.

To describe the source of the pessimism, we need to introduce the canonicalization.
As mentioned in the text, for a particular tree $T$, there are many possible paths $T=T_0, \dots, T_M$.
Phyloformer 2 defines a deterministic ``correct'' path $P(T, \vec b)=T_0, \dots, T_M$ where $\vec b$ are the branch lengths and the cherry with the shortest branch lengths is removed every step (see \citet{Blassel2025-vq} for details).
Then we can write an ELBO
\begin{equation}\label{eqn:pf2 ELBO}
    \mathbb E_{T, X_{1:N}\sim p}\log q_\theta(T\mid X_{1:N})\geq \mathbb E_{T, X_{1:N}, \vec b\sim p}\log q_\theta(P(T, \vec b)\mid X_{1:N}).
\end{equation}
In Phyloformer 2 their goal was to predict $\log q_\theta(T, \vec b\mid X_{1:N})$ so the ELBO Eqn.~\ref{eqn:pf2 ELBO} could be made tight if the model learned $q_\theta(T_0, \dots, T_M, \vec b\mid X_{1:N})=0$ if $(T_0, \dots, T_M)\neq P(T_0, \vec b)$.
In our case however, $P(T, \vec b)$ is a random variable and the best prediction is
$q_\theta(T_0, \dots, T_M\mid X_{1:N})=p(T_0\mid X_{1:N})p(P(T_0, \vec b)=(T_0, \dots, T_M)\mid T_0, X_{1:N}).$
The second term creates an irreducible gap in the ELBO of size $\mathbb E_{T, X_{1:N}\sim p}\mathrm{Entropy}(P(T, \vec b)\mid T, X_{1:N})$.

To ablate this effect, we replace the true branch lengths used to generate the tree with branch lengths predicted from $T$ and $X_{1:N}$, i.e. we replace $P(T, \vec b)$ with a $P(T, X_{1:N})$ so it is possible for the ELBO to be tight; i.e. $\mathrm{Entropy}(P(T, X_{1:N})\mid T, X_{1:N})=0$.
To do so, we first estimate the distance between every pair of DNA sequences $X_n$ and $X_{n'}$ as $b_{n, n'}=-\frac 3 4\times\log(1-\frac 4 3\times \text{fraction mismatches})$ -- this estimator converges to the exact distance under the JC model when $L\to\infty$.
Then, traversing the tree from the bottom-up, for every node $v$ with children $\ell$ and $r$, and other node $w$, we estimate $b_{v, w}=\frac{1}{2}(b_{\ell, w}+b_{r, w}-b_{\ell, r})$.
Given these branch lengths $\vec b(T_0, X_{1:N})$, we then construct $P(T, \vec b(T_0, X_{1:N}))$

\paragraph{Hypothesis 2: Our architecture improvements explain the improved performance}

Here we port the improvements we made to our architecture to Phyloformer 2.
In particular, we add $\mu$P-rescaled per-layer learning rates, initialization, and attention normalization~\citep{Yang2022-sy}, CompleteP scaled residual layers~\citep{Dey2025-uv}, EMA $0.9998$, gradient-norm clip $500$, and \texttt{bf16-mixed} precision.
We scan learning rates $0.001, 0.003, 0.01, 0.03, 0.1$ each for 12 hours, and find that $0.003$ is optimal.
We see the loss is at each step is similar to before, but the throughput increased roughly by $2\times$.

\paragraph{Ablation results}

Fig.~\ref{fig:pf2 ablation} shows the results of the ablations described above compared to the original Phyloformer 2 and \bu{} results.
We see hypotheses 1 and 2 are not true: evaluating a non-pessimistic ELBO, or porting our architectural improvements to Phyloformer 2 did not majorly affect training.
The model trained with a non-pessimistic ELBO may have done slightly worse than the original Phyloformer 2 because the canonicalization we propose may have been harder to learn.

On the other hand, as in App.~\ref{app: ablate math}, we see Rao-Blackwellizing the objective results in a substantially better fit.
However there is still a significant gap to our \bu{} model, which likely comes from our other contribution, the increased flexiblity.

\begin{figure}[H]
    \centering
    \includegraphics[width=0.65\linewidth]{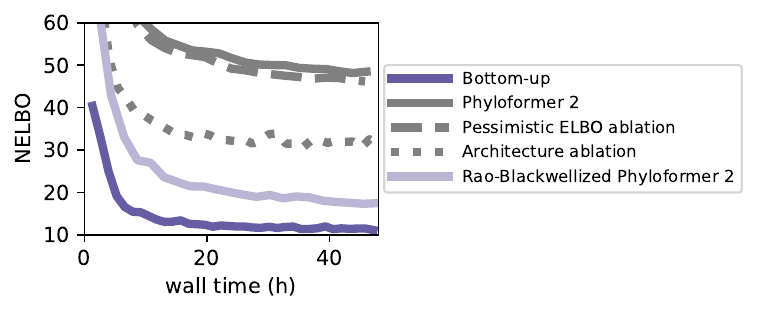}
    \caption{\textbf{Our improvements over Phyloformer 2 came from our improved objectives and the flexibility of our models.}
    We ablate the architectural improvements we've made and the form of the canonicalized objective.
    These changes don't make a large change, while Rao-Blackwellize does, but not enough to close the gap to our model.}
    \label{fig:pf2 ablation}
\end{figure}

\subsection{Posterior calibration}\label{app: posterior calibration}

On the JC data we have access to a computationally cheap likelihood $p(X_{1:N}\mid T)$ and can approximate the ground truth posterior with MrBayes.
This allowed us to quantify an approximate average KL between our inferred posteriors $q_\theta(T_0\mid X_{1:N})$ and the true posterior.
Here we ask whether \textit{samples} from $q_\theta(T_0\mid X_{1:N})$ look like samples from the true posterior.
We also ask how much this depends on using our synthetic data versus real data which can be out-of-distribution.
This extends the analysis in Sec.~\ref{sec: post calib}.

We consider classical evolutionary datasets labeled DS1-8~\citep{Lakner2008-dd} numbered in order of increasing $N$ and, for each dataset, simulate JC data with the same $N$ and $L$.
Then we compare our inferences to those of a long MrBayes run on these $16$ data using 3 methods.
\begin{itemize}
    \item \textit{How far is $q_\theta(T_0\mid X_{1:N})$ from the true posterior?} We compare them with MMD.
    MMD measures the average distance between trees sampled from separate distributions and subtracts the average distance between trees sampled from the same distribution, with the best MMD score of 0 achieved when the two distributions are identical. See App.~\ref{app: mmd} for more details. 
    \item \textit{If we make conclusions about our confidence of evolutionary events $q_\theta(T_0\mid X_{1:N})$ will they match the confidence of those from the true posterior?}
    In a tree, every parent-child branch partitions the sequences $X_{1:N}$ into two groups: those that are descendants of the parent, and everything else.
    Whether two species often belong to the same split or different split determines their related-ness.
    Therefore the posterior likelihood of a split is regularly used to make conclusions about our confidence of species-relatedness.
    So, for every possible split that appears in any observed tree, we count how likely it is to appear in the MrBayes posterior versus samples from our posterior.
    \item \textit{Do samples from $q_\theta(T_0\mid X_{1:N})$ look like those from the true posterior?} We plot samples for visual comparison.
\end{itemize}

In Fig.~\ref{fig:all mmds} we plot MMDs for the real and synthetic DS 1-8.
We see we substantially outperform Phyloformer 2.
We also see that our MMDs are near $0$ for the synthetic data, but all models find the real data much more challenging.

\begin{figure}[H]
    \centering
    \includegraphics[width=0.9\linewidth]{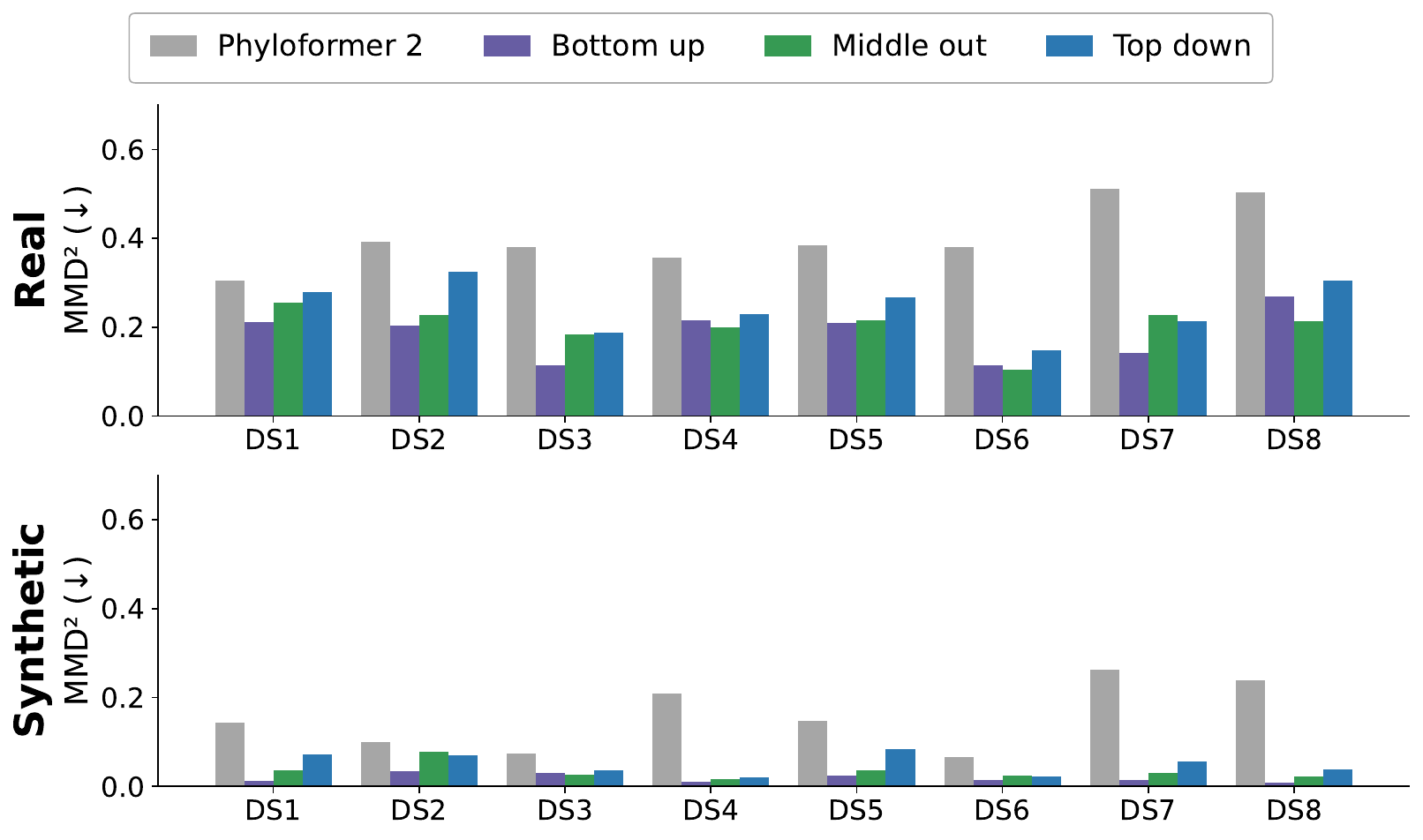}
    \caption{\textbf{Our model matches the MrBayes posterior distribution.}
    We plot MMDs on DS1-8 between $1000$ samples from each model and MrBayes.
    The top row is a reproduction of Fig.~\ref{fig:mmd} included for comparison.}
    \label{fig:all mmds}
\end{figure}

In Fig.~\ref{fig:all calib} we plot the posterior probabilities of each observed split from MrBayes and our models (each point is a split).
We report the Spearman correlation ($\rho$) as well as the coverage of likely splits (we take all splits with a posterior probability $\geq 0.9$ under MrBayes and report the average posterior probabilities of these splits under our model).
We see all models are under-confident (below the diagonal line), likely because the likelihood objective encourages models to spread-out.
This is in particular true for the real data, suggesting that our models' errors come largely from abstaining to make confident predictions on OOD data.
Again, we see our models are much more calibrated than Phyloformer 2, and that models are less calibrated on real data.
However, we see that the coverage of our models remains high: most high probability splits in MrBayes are also high probability under our models.

\begin{figure}[H]
    \centering
    \includegraphics[width=0.99\linewidth]{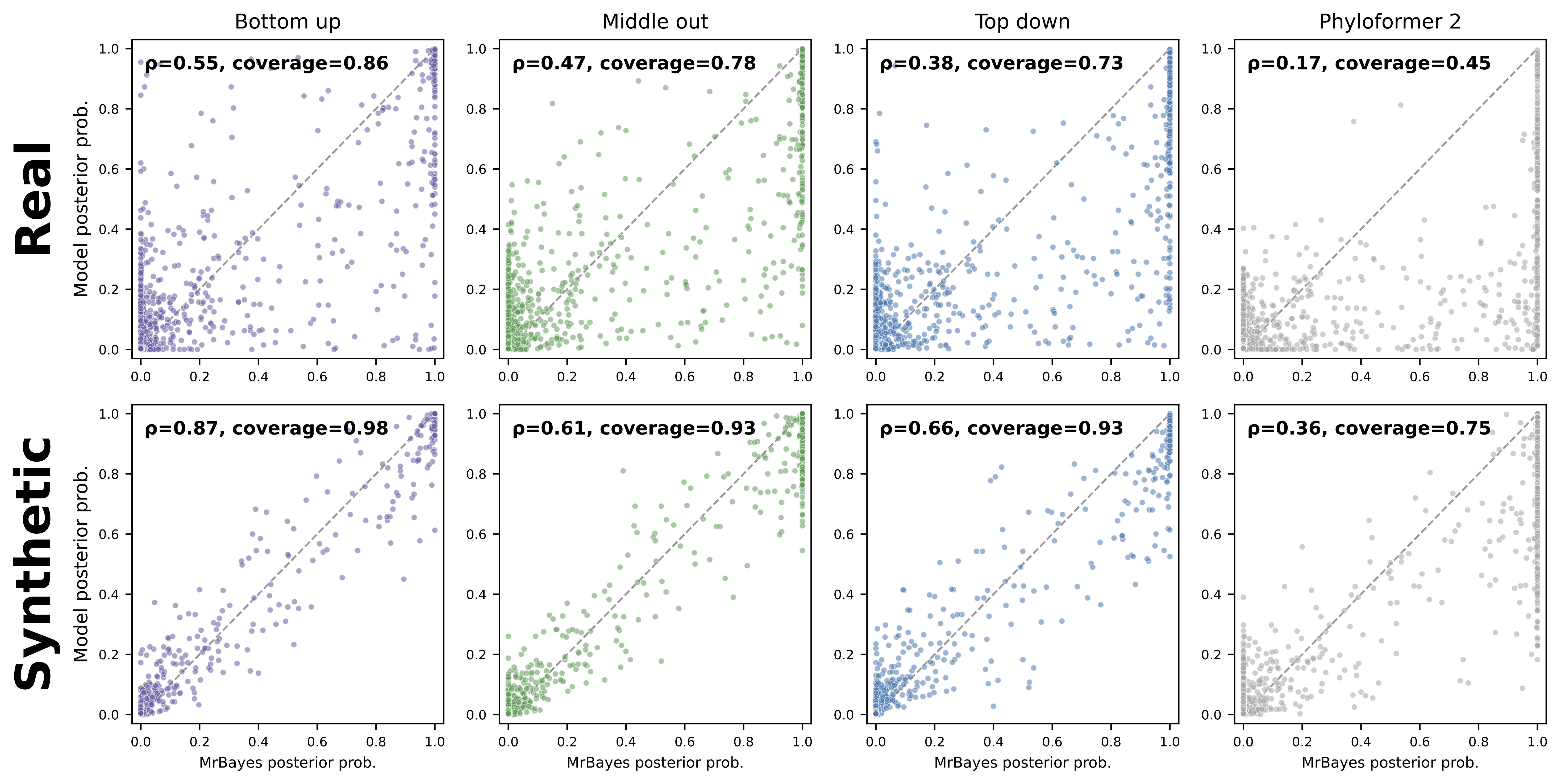}
    \caption{\textbf{Our models have accurate posterior split probabilities.}
    We plot posterior probabilities of observed splits on all of DS1-8 between $1000$ samples from each model and MrBayes.
    }
    \label{fig:all calib}
\end{figure}

In Fig.~\ref{fig:ds3 real trees} and \ref{fig:ds3 synth trees} we plot the consensus and sample trees from each method for DS3 real and synthetic data.
This is a setting in which the uncertainty was low, making it easier to visualize.
We see, as above, that our models are under-confident and perform better on synthetic data.
Trees from our models however clearly better respect the original tree than those from Phyloformer 2.

\begin{figure}[h]
    \centering
    \includegraphics[width=0.88\linewidth]{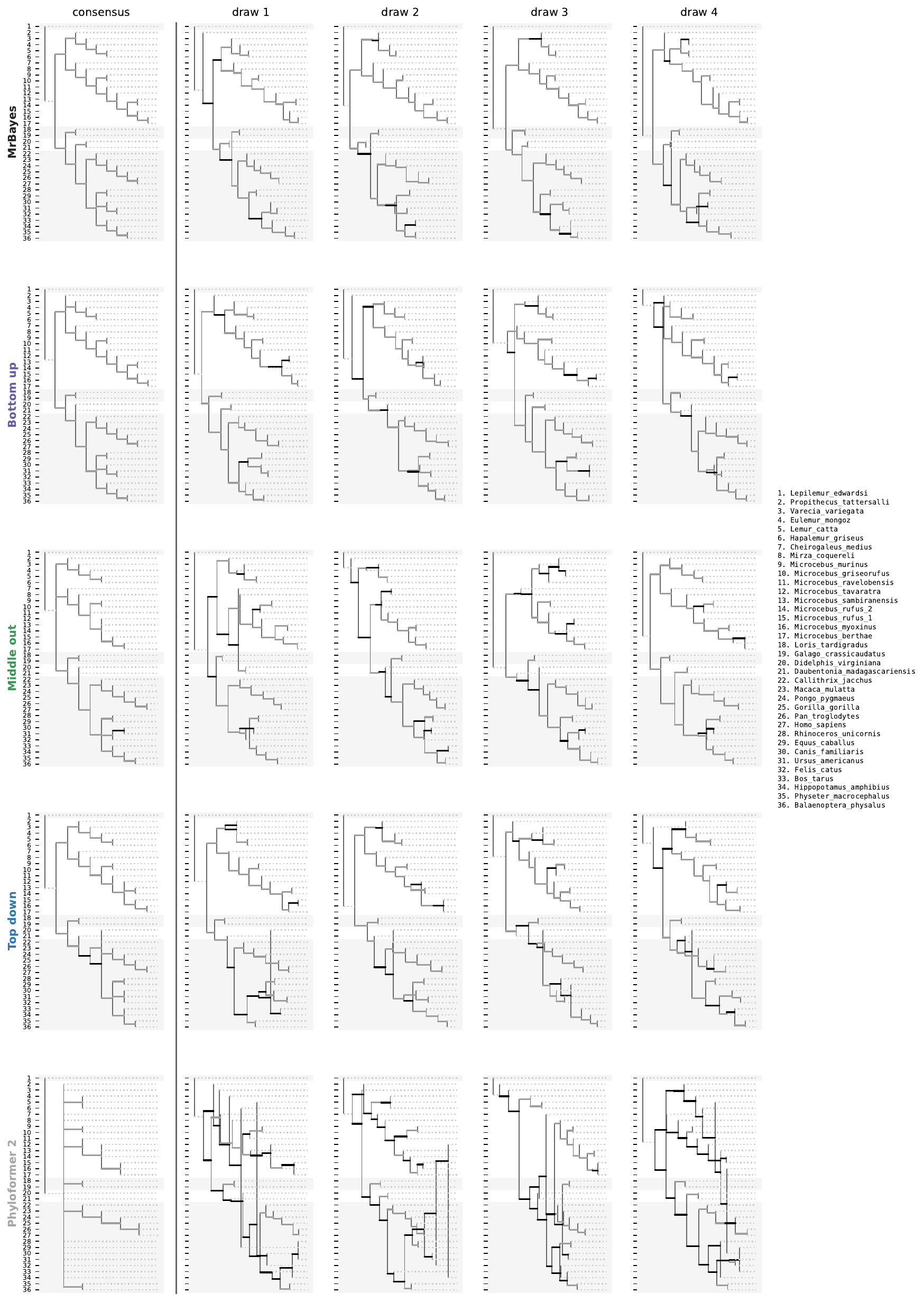}
    \caption{Consensus and inferred trees on DS3.
    We plot the consensus over 500 sequences.
    We colour branches that don't appear in the MrBayes consensus black.
    We root trees by picking a fixed random out-group among the nodes.
    Clades are shaded grey to help in comparing trees.
    Species are ordered to allow us to show the MrBayes consensus tree as a planar tree.}
    \label{fig:ds3 real trees}
\end{figure}

\begin{figure}[h]
    \centering
    \includegraphics[
            width=0.7\linewidth,
            trim=0 0 50 0,
            clip
        ]{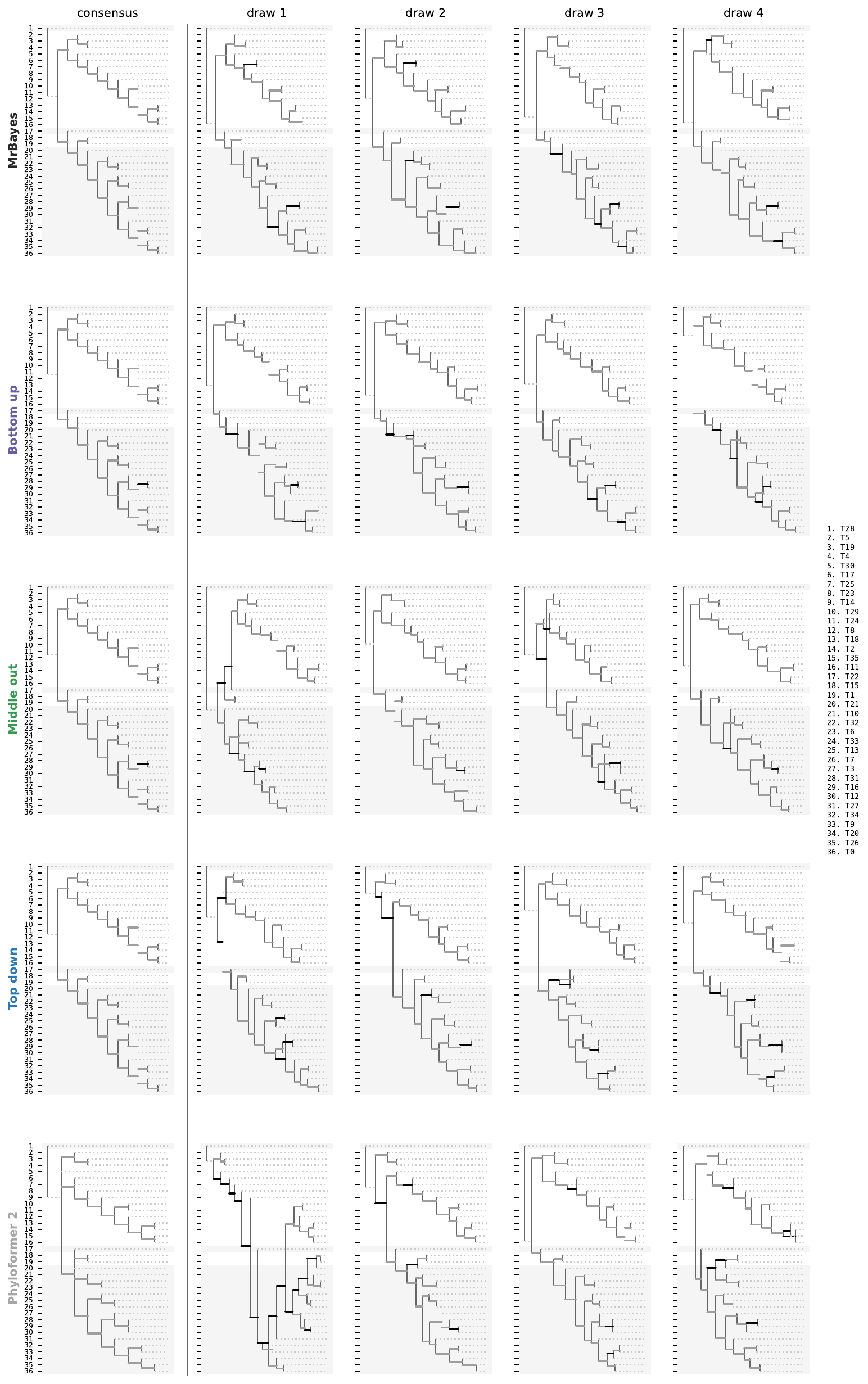}
    \caption{Consensus and inferred trees on a synthetic DS3.
    We plot the consensus over 500 sequences.
    We colour branches that don't appear in the MrBayes consensus black.
    We root trees by picking a fixed random out-group among the nodes.
    Clades are shaded grey to help in comparing trees.}
    \label{fig:ds3 synth trees}
\end{figure}

\subsection{Model sampling timing comparisons}
\label{app: timing results}

Our \bu{} and \mo{} models samples scale as $N$ steps of $\times O(N^4+NL^2)$ per step.
Phyloformer 2 in principle scales the same way but only requires one step.
Despite the extra cost of our model, it is cheaper than the cheapest default settings of IQ-Tree for most parameter settings, making it appropriate for use in standard phylogenetics workflows.
The exception is our \td{} model, which needs many more iterative steps to generate a tree, between $\Theta(N\log N)$ and $\Theta(N^2)$ depending on the depths of the nodes in the tree (average $\Theta(N^{3/2})$ for PDA).
The extra compute our models spend during sampling allow them to represent much more flexible distributions, bringing these resources to bear on this problem.

\begin{figure}[H]
    \centering
    \includegraphics[width=1\linewidth]{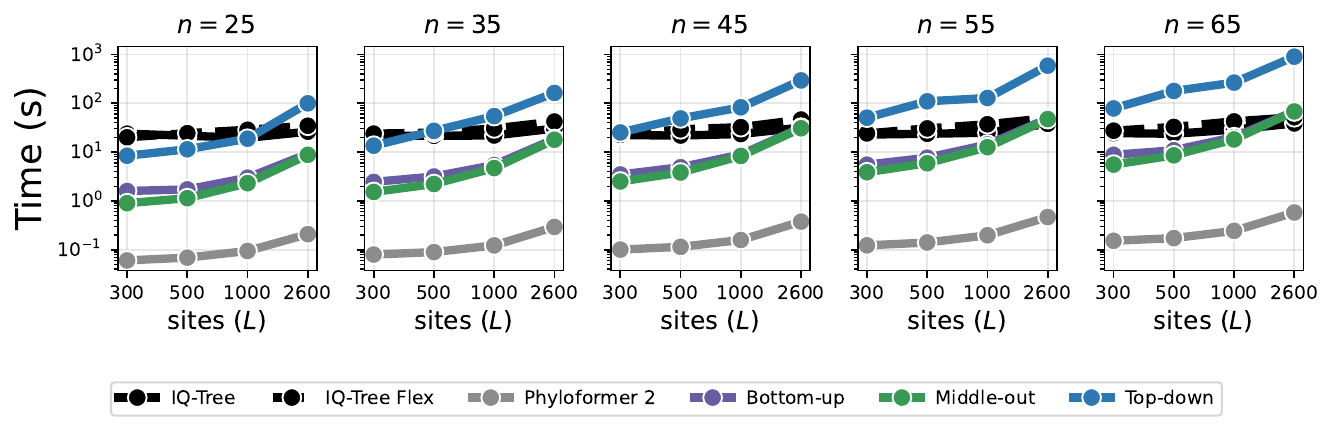}
    \caption{Comparison of sampling runtimes for our models and baselines on an A100 GPU.}
    \label{fig:placeholder}
\end{figure}

\end{document}